\documentclass[conference]{IEEEtran}  

\IEEEoverridecommandlockouts                              

\usepackage[numbers]{natbib}
\usepackage{amsthm}
\usepackage{times}
\usepackage{multicol}
\usepackage[bookmarks=true]{hyperref}
\usepackage{xcolor}
\usepackage{hyperref}
\usepackage{amsmath, amssymb}
\usepackage{amsfonts}
\usepackage{graphicx}
\usepackage{siunitx}
\usepackage{standalone}
\usepackage{booktabs}
\usepackage[ruled,vlined,linesnumbered,noend]{algorithm2e}
\usepackage[para,online,flushleft]{threeparttable}
\usepackage{mdframed}
\usepackage{fancyvrb,multirow}
\usepackage{soul}
\usepackage{dsfont,mathabx}
\usepackage{array, booktabs}
\usepackage{subfigure}
\usepackage{booktabs}
\usepackage{makecell}
\usepackage{threeparttable}
\usepackage{tabularx}
\usepackage{booktabs}
\usepackage{amsmath,amsfonts,amsthm,bm} %
\usepackage[labelfont=bf,format=plain]{caption}
\usepackage[T1]{fontenc}
\usepackage{tikz}
\usepackage{pgfplotstable}
\usepackage{pgfplots}

\newtheorem{theorem}{Theorem}

\newtheorem{lemma}{Lemma}
\theoremstyle{definition}

\newtheorem{definition}{Definition}
\newtheorem{remark}{Remark}

\newcommand{\change}[1]{\textcolor{black}{#1}}

\title{Collaborative Planar Pushing of Polytopic Objects with Multiple Robots in Complex Scenes}

\author{Zili Tang, Yuming Feng and Meng Guo$^*$ \\ College of Engineering, Peking University
  \thanks{$^*$Corresponding author: Meng Guo, {\tt\small meng.guo@pku.edu.cn}.}

}

\begin{document}
\maketitle
\thispagestyle{empty}
\pagestyle{empty}


\begin{abstract}
  Pushing is a simple yet effective skill for robots to interact with and further
  change the environment.
  Related work has been mostly focused on utilizing 
  it as a non-prehensile manipulation primitive for a
  robotic manipulator.
  However, it can also be beneficial for low-cost mobile robots that are not equipped with a manipulator.
  This work tackles the general problem of controlling a team of mobile robots
  to push collaboratively polytopic objects within complex obstacle-cluttered environments.
  It incorporates several characteristic challenges for contact-rich tasks such as the
  hybrid switching among different contact modes and under-actuation due to constrained contact forces.
  The proposed method is based on hybrid optimization over a sequence of possible modes 
  and the associated pushing forces,
  where (i) a set of sufficient modes is generated with a 
  multi-directional feasibility estimation,
  based on quasi-static analyses for general objects and any number of robots;
  (ii) a hierarchical hybrid search algorithm is designed to iteratively
  decompose the navigation path via arc segments 
  and select the optimal parameterized mode;
  and (iii) a nonlinear model predictive controller is proposed to track the 
  desired pushing velocities
  adaptively online for each robot.
  The proposed framework is complete under mild assumptions.
  Its efficiency and effectiveness are validated in high-fidelity simulations and hardware experiments.
  Robustness to motion and actuation uncertainties is also demonstrated.
  More examples and videos are available at \url{https://zilitang.github.io/Collaborative-Pushing}.
\end{abstract}

\section{Introduction}\label{sec:intro}
Non-prehensile manipulation skills such as pushing are essential
when humans interact with objects, as an important complementary
to prehensile skills such as stable grasping.
Existing work can be found that exploits this aspect for robotic systems,
mostly for a single manipulator within simple environments,
e.g.,~\citet{goyal1989limit, lynch1992manipulation, hogan2020reactive, xue2023guided}.
However, pushing can also be beneficial for low-cost mobile robots
that are {not} equipped with a manipulator,
e.g., obstacles could be pushed away,
or target objects could be pushed to destinations.
As illustrated in Fig.~\ref{fig:intro},
a fleet of such robots could be more efficient
by collaboratively and simultaneously pushing the same object
at different points with different forces.
Theses combinations generate a rich set of pushing modes,
e.g., long-side, short-side, diagonal, and caging,
which lead to different motions of the object such translation and rotation,
however with varying quality.
Despite of its intuitiveness, the collaborative pushing task incorporates
several challenges that are typical for contact-rich tasks,
i.e., the hybrid dynamic system due to switching
between different contact modes,
under-actuation due to constrained contact forces,
and tight kinematic and geometric constraints
(such as narrow passages and sharp turns),
yielding it a difficult problem not only for
modeling and planning, but also for optimal control.

\begin{figure}[t]
  \centering
  \includegraphics[width=0.98\linewidth]{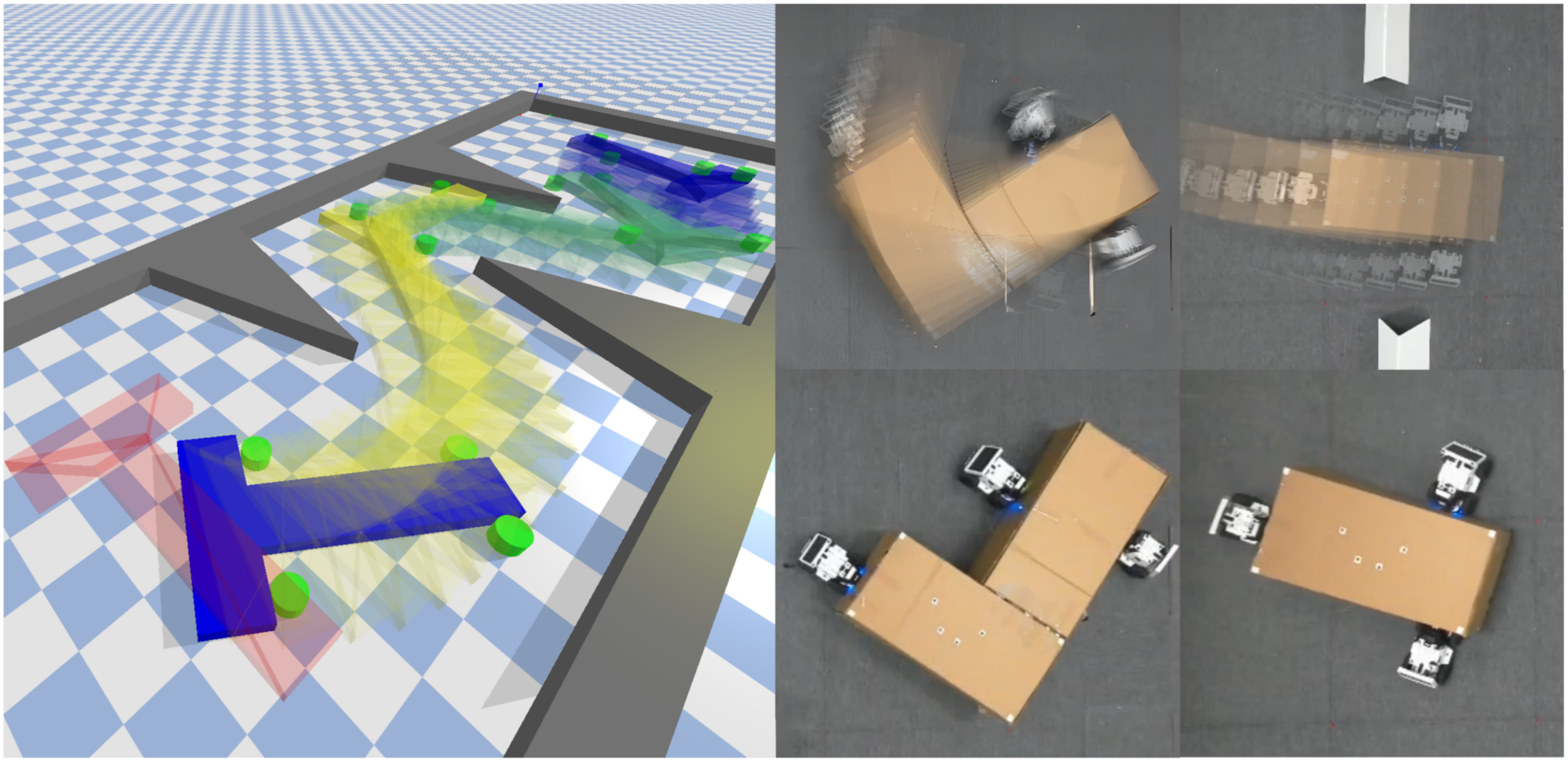}
  \caption{Snapshots of the collaborative pushing
  task in simulation and hardware experiments.}
  \label{fig:intro}
  \vspace{-3mm}
\end{figure}

\subsection{Related Work}\label{subsec:intro-related}

The task and motion planning problem
for prehensile manipulation is the predominant research direction with a large body of literature.
It has a strong focus on the grasping policies for different objects,
and the sequence of manipulating these objects for a long-term manipulation task,
such as assembly in~\citet{toussaint2015logic,guo2021geometric},
re-arrangement in~\citet{kim2019learning} and construction in~\citet{22-hartmann-TRO}.
On the other hand, non-prehensile planar manipulation is a challenging problem
on its own for model-based planning and control without stable grasping.
The \emph{single-pusher-single-slider} system has been introduced as the classic
model from~\citet{goyal1989limit},
where both the discrete decision for interaction modes and continuous optimization
for control inputs are essential.
However, the inclusion of discrete variables drastically increases the planning complexity,
i.e., exponential to the planning horizon.
For a small number of pre-defined modes such as sticking and sliding,
this decision can be directly formulated as an integer variable in
the motion planning scheme from~\citet{hogan2020reactive}
thus solved via integer programs.
In addition, various search methods can be applied to explore the statespace
e.g., via sampling in~\citet{wang2021learning},
multi-bound tree search in~\citet{toussaint2018differentiable},
and backtracking in~\citet{kaelbling2011hierarchical}.
To further alleviate the complexity,
human demonstrations are used in~\citet{xue2023guided,le2021learning}
to guide the choice of contact modes,
while learning-based methods are proposed
to predict the contact sequence in~\citet{simeonov2021long} or
even raw inputs in~\citet{zeng2018learning} for short-term tasks.
However, these methods are mostly developed for a single manipulator
and fix-shaped objects,
of which the control and planning techniques are not applicable to
multi-robot fleets.

Collaborative pushing belongs to a larger application domain of multi-robot systems:
cooperative object transportation, see~\citet{tuci2018cooperative},
which includes pushing, grasping and caging as different behaviors.
The pioneer work in~\citet{kube1997task} designs a multi-layer
state machine for~$5$ robots to
push a rectangular box without direct communication in free space.
A simple yet effective strategy is proposed in~\citet{chen2015occlusion}
to transport any convex object in obstacle-free space,
i.e., the robots only push the object at positions
where the direct line to the goal is occluded by the object.
A combinatorial-hybrid optimization problem is formulated
in~\citet{tang2023combinatorial} for several objects
but with pre-defined pushing primitives.
Fleets of heterogeneous robots are deployed in~\citet{vig2006multi} to clear
movable obstacles via pushing, however via a pre-defined sequence of steps.
Model-free approaches can be found in~\citet{xiao2020learning}
that map directly the scene to control inputs but lacks the theoretical guarantees on completeness.
Other related work neglects the control aspect and focuses only on the
high-level task assignment, e.g.,~\citet{garcia2013scalable}.

\subsection{Our Method}\label{subsec:intro-our}
This work addresses the general problem of controlling a team of mobile robots
to push collaboratively a polytopic object to a goal position in a complex environment,
without any pre-defined contact modes or primitives.
To begin with,
a multi-directional feasibility analysis is conducted for a quasi-static pushing condition,
given the desired object motion and a parameterized contact mode.
Then, the notion of sufficient contact modes is introduced for pushing the object
along a desired trajectory, based on its finite segmentation into arcs.
Thus, a hierarchical hybrid search algorithm is proposed for
finding a feasible sequence of modes and the associated force profiles.
It iteratively decomposes the desired trajectory via key-frames and expands the search tree
via generating suitable modes between consecutive frames,
in order to minimize the aforementioned feasibility loss, transition cost and control efforts.
Lastly, a nonlinear model predictive controller (MPC) is proposed to track these arc segments
online with the given mode, while accounting for motion uncertainties and undesired slips adaptively.
Theoretical guarantees on completeness and performance are provided.
Extensive high-fidelity simulations and hardware experiments are conducted
to validate its efficiency and robustness.

Main contribution of this work is two-fold:
(i) the novel hierarchical hybrid-search algorithm
for the multi-robot planar pushing problem of general polytopic objects,
without any pre-defined contact modes or primitives;
(ii) the theoretical condition for feasibility and completeness,
w.r.t. an arbitrary number of robots, any polytopic shape and any desired object trajectory.
To the best of our knowledge, this is the first work that provides such results.



\section{Problem Description}\label{sec:problem}

\subsection{Model of Workspace and Robots}\label{subsec:ws}

Consider a team of~$N$ robots~$\mathcal{R}\triangleq \{R_n,\,n\in \mathcal{N}\}$
that collaborate in a shared 2D workspace~$\mathcal{W}\subset \mathbb{R}^2$,
where~$\mathcal{N}\triangleq \{1,\cdots,N\}$.
The robots are homogeneous with a circular or rectangular shape.
The state of each robot at time~$t\geq 0$ is defined by its position
and orientation,
i.e.~$\mathbf{x}_n(t)\triangleq(x_{n},y_{n})$ is the center
coordinate,~$\psi_{n}$ is the rotation,~$\mathbf{v}_n(t)\triangleq(v_{nx},v_{ny})$
is the linear velocity,
and~$\omega_{n}$ is the angular velocity,
$\forall n \in \mathcal{N}$.
Let~$R_n(t)\subset \mathcal{W}$ also denote the area occupied by robot~$R_n$
at time~$t\geq 0$.
Moreover, each robot has a feedback controller for velocity tracking.
Thus, the robots follow a simple first-order dynamics in the free space,
i.e., $(\dot{\mathbf{x}}_n,\dot{\psi}_n)= (\mathbf{v}_n,\omega_n)\triangleq \mathbf{u}_n$.
In addition, the workspace is cluttered with a set of
static obstacles, denoted by~$\mathcal{O} \subset \mathcal{W}$.
Their shape and position are known a priori.
Thus, the freespace is defined
by~$\widehat{\mathcal{W}}\triangleq \mathcal{W}\backslash \mathcal{O}$.

Lastly, there is a single target object~$\Omega\subset \widehat{\mathcal{W}}$,
of which the shape is an arbitrary polygon formed by
$V\geq 3$ ordered vertices ${p_1p_2\cdots p_V}$.
Similarly, its state at time~$t\geq 0$ is defined by its position
and orientation,
i.e.~$\mathbf{x}(t)\triangleq(x,y)\in \widehat{\mathcal{W}}$ is the center of mass
and~$\psi \in [-\pi,\pi]$ is the rotation,
$\mathbf{v}(t)\triangleq(v_x,v_y)$ is the linear velocity
and~$\omega$ is the angular velocity,
and its occupied area is denoted by~$\Omega(t)$.
For the ease of notation,
let~$\mathbf{s}(t)\triangleq(x,y,\psi)\in \mathcal{S}$
denote its full state
and~$\mathbf{p}(t)\triangleq (v_x,v_y,\omega)$
denote its generalized velocity.
In addition,
its intrinsics are known a priori, including its mass~$M$,
the moment of inertia~$\mathcal{I}$ about the center of mass,
the pressure distribution at bottom surface,
the coefficient of lateral friction~$\mu_c>0$,
and the coefficient of ground friction~$\mu_s>0$.



\begin{figure*}
  \centering
  \includegraphics[width=1\linewidth]{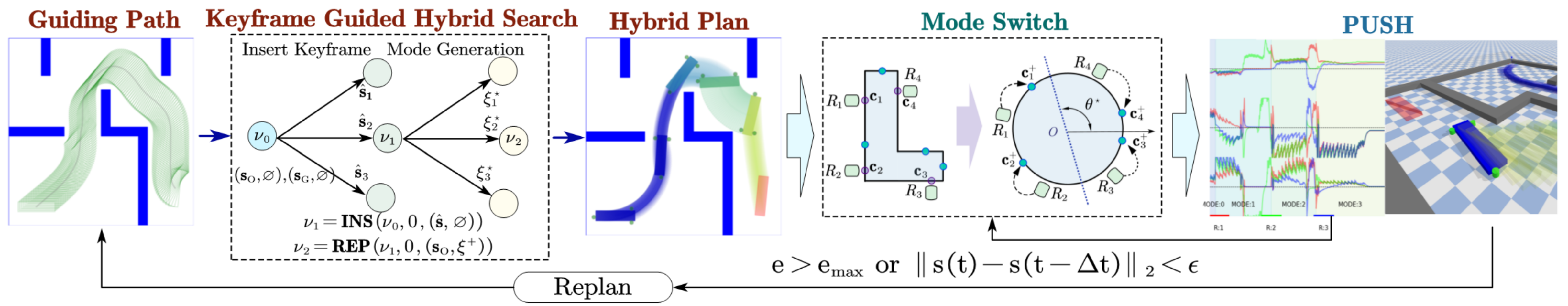}
  \vspace{-4mm}
  \caption{\change{Overall framework, including the hybrid plan generation in Sec.~\ref{subsec:hybrid}
      and the online execution in Sec.~\ref{subsec:control}}.}
  \label{fig:overall}
  \vspace{-2mm}
\end{figure*}

\subsection{Interaction Modes and Coupled Dynamics}\label{ss:interaction_mode}
\def\sss{\scriptscriptstyle{}}
The robots can collaboratively move the target by
making contacts with the target at different contact points,
called \emph{interaction modes}.
More specifically, an interaction mode is defined
by~$\xi\triangleq \mathbf{c}_1\mathbf{c}_2\cdots\mathbf{c}_N$,
where~$\mathbf{c}_n\in \partial \Omega$ is the contact point on the boundary
of the target for the $n$-th robot,
$\forall n\in \mathcal{N}$.
Since the set of contact points are \emph{not} pre-defined,
the complete set of all interaction modes
is potentially infinite, denoted by~$\Xi$.
Moreover, if any robot~$R_n$ is not in contact with the target,
it is called a transition mode, denoted by~$\xi_0$.
Thus, given a particular interaction mode,
the robots can apply pushing forces in different directions
with different magnitude.
Denote by~$\mathbf{f}_1 \mathbf{f}_2 \cdots \mathbf{f}_N$,
where~$\mathbf{f}_n\in \mathbb{R}^2$ is the contact force of robot~$R_n$ at
contact point~$\mathbf{c}_n$, $\forall n\in \mathcal{N}$.
Furthermore, each force~$\mathbf{f}_n$ can be decomposed
in the directions of the normal vector~$\mathbf{n}_n$ and the tangent vector~$\bm{\tau}_{n}$
w.r.t.~the target surface at the contact point~$\mathbf{c}_n$,
i.e.,~$\mathbf{f}_n\triangleq \mathbf{f}^{\texttt{n}}_n + \mathbf{f}^{\texttt{t}}_n
\triangleq f^{\texttt{n}}_n \mathbf{n}_n + f^{\texttt{t}}_n \bm{\tau}_n$.
Due to the Coulomb law of friction see~\cite{kao2016contact},
it holds that:
\begin{equation}\label{eq:force-limit}
0\leq f^{\texttt{n}}_n \leq f_{n,\max};
\; 0\leq |f^{\texttt{t}}_n| \leq \mu_c f^{\texttt{n}}_n,
\end{equation}
where~$f_{n,\max}> 0$ is the maximum force each robot can apply;
and~$\mu_c$ is the coefficient of lateral friction defined earlier.
Then, these decomposed forces can be re-arranged by:
\begin{equation}\label{eq:F-xi}
\mathbf{F}_\xi\triangleq (\mathbf{F}^{\texttt{n}}_{\xi},\,\mathbf{F}^{\texttt{t}}_{\xi})
\triangleq (f^\texttt{n}_{1},\cdots,f^\texttt{n}_{\sss{N}},
f^\texttt{t}_{1},\cdots,f^\texttt{t}_{\sss{N}})\in\mathbb{R}^{\sss{2N}},
\end{equation}
and further~$\mathcal{F}_{\xi}\triangleq \{\mathbf{F}_{\xi}\}$
denotes the set of all forces within each mode~$\xi\in \Xi$.
Furthermore, the combined generalized
force~$\mathbf{q}_{\xi}\triangleq (f_x^\star,f_y^\star,m^\star)$
as also used in~\citet{lynch1992manipulation} is given by:
\begin{equation}\label{eq:generalized_force}
    (f_x^\star,\, f_y^\star)\triangleq \textstyle\sum_{n=1}^N{\mathbf{f}_n};\;
    m^\star\triangleq \sum_{n=1}^N{(\mathbf{c}_n-\mathbf{x}_n)\times \mathbf{f}_n},
\end{equation}
where~$\times$ is the cross product and~$m^\star$ is the resulting torque from all robots.
It can be written in matrix form~$\mathbf{q}_{\xi}\triangleq \mathbf{J}\mathbf{F}_\xi$,
where $\mathbf{J}\triangleq \nabla_{\sss{\mathbf{F}_\xi}} \mathbf{q}_{\xi}$ is
a~$3\times 2N$ Jacobian matrix.
Similarly, let $\mathrm{Q}_{\xi}\triangleq \{\mathbf{q}_{\xi}\}$
denote the set of all \emph{allowed} combined generalized forces within each mode~$\xi\in \Xi$.
Given~$\mathbf{q}_{\xi}\in \mathrm{Q}_{\xi}$, the coupled dynamics of the robots and the target
under mode $\xi$ can be determined as follows:
\begin{subequations}\label{eq:multi_body}
  \begin{align}
    &\mathbf{M}\dot{\mathbf{p}}=\mathbf{q}_{\xi}+\mathbf{q}_{\mu}
    \triangleq \zeta (\mathbf{p},\mathbf{u}_{\mathcal{N}})+\eta(\mathbf{p});
    \label{eq:target_equation}\\
    &\dot{\mathbf{x}}_n =\mathbf{v}_n=\mathbf{v}
    +\bm{\omega}\times(\mathbf{c_n}-\mathbf{x}),
    \quad \forall n\in \mathcal{N}; \label{eq:robot_equation}
  \end{align}
\end{subequations}
where~$\mathbf{M} \triangleq \text{diag}(M,M,\mathcal{I})$;
$\mathbf{q}_{\xi}\triangleq \zeta(\mathbf{p},\mathbf{u}_{\mathcal{N}})$
is the generalized pushing force in~\eqref{eq:generalized_force}
given the combined control inputs~$\mathbf{u}_{\mathcal{N}}\triangleq\mathbf{u}_1 \cdots \mathbf{u}_{N}$
and the velocity of the target~$\mathbf{p}$;
and~$\mathbf{q}_{\mu}\triangleq \eta(\mathbf{p})=(f_x,f_y,m)$
is the ground friction, determined by the velocity~$\mathbf{p}$
and intrinsics of the target.
Namely,~\eqref{eq:target_equation} describes how the target moves
under the external forces $\mathbf{q}_{\xi}$ and $\mathbf{q}_\mu$
within mode~$\xi$;
and~\eqref{eq:robot_equation} describes how
the robot velocity is determined by the target velocity
under the non-slipping constraints.
\begin{remark}\label{remark:analysis}
Due to the complex nature of the multi-body physics under contact,
both the pushing forces~$\mathbf{q}_{\xi}$ and friction~$\mathbf{q}_{\mu}$
in~\eqref{eq:target_equation} lack an analytical form,
yielding difficulties in state prediction and control.
However, it is certain that~$\mathbf{q}_{\xi}$
must be \emph{allowed} in mode~$\xi$,
i.e.,~$\mathbf{q}_{\xi} \in \mathrm{Q}_{\xi}$.
\hfill $\blacksquare$
\end{remark}

\subsection{Problem Statement}\label{subsec:objective}
The planning objective is to compute a \emph{hybrid plan},
including a target trajectory~$\mathbf{s}(t)$,
a sequence of interaction modes~$\mathbf{\xi}(t)$ and the associated pushing
forces~$\mathbf{q}_{\xi}(t)$,
such that the target is moved from any initial state~$\mathbf{s}_{0}$
to a given goal state~$\mathbf{s}_{\texttt{G}}$.
Meanwhile, the robots and the target should avoid collision with all obstacles at all time.
More precisely, it is formulated as a hybrid optimization problem:
\begin{equation}\label{eq:problem}
\begin{split}
&\underset{{\{\mathbf{\xi}(t),\mathbf{q}_{\xi}(t),\mathbf{s}(t)\},T}}{\textbf{min}}\;
\Big{\{}T+\alpha \sum_{t=0}^{T} J\big(\xi(t),\mathbf{q}_{\xi}(t),{\mathbf{s}}(t)\big)\Big{\}} \\
&\textbf{s.t.}\quad \mathbf{s}(0)= \mathbf{s}_{0},\, \mathbf{s}(T)= \mathbf{s}_{\texttt{G}}; \\
& \quad \quad \, \Omega(t)\subset\widehat{\mathcal{W}},\,
  R_n(t)\subset\widehat{\mathcal{W}},\; \forall n\in \mathcal{N},\, \forall t \in \mathcal{T}; \\
& \quad \quad \, \xi(t)\in\Xi,\,
  \mathbf{q}_{\xi(t)} \in \mathrm{Q}_{\xi(t)},\; \forall t \in \mathcal{T};\\
\end{split}
\end{equation}
where~$T>0$ is the task duration to be optimized;
$\mathcal{T}\triangleq \{0,1,\cdots,T\}$;
$J:\Xi \times \mathbb{R}^3 \times \mathbb{R}^{3} \rightarrow \mathbb{R}_{+}$ is
a given function to measure the feasibility, stability and control cost
of choosing a certain mode and the applied forces given the desired target trajectory;
and~$\alpha>0$ is the weight parameter to balance the task duration
and the control performance.

\begin{remark}\label{remark:J}
  The design of function~$J(\cdot)$ is non-trivial
  and technically involved, thus explained further in the sequel.
  \hfill $\blacksquare$
\end{remark}

\begin{remark}\label{remark:complexity}
  The hybrid search space of~\eqref{eq:problem} above is
  over the timed sequence of pushing modes, forces and the target trajectory,
  which has dimension~$(N+6)\cdot T_{\max}$ with~$T_{\max}$ being the maximum duration.
  \hfill $\blacksquare$
\end{remark}

\begin{remark}\label{remark:modes}
  Existing work in~\citet{wang2006multi,goyal1989limit,chen2015occlusion,tang2023combinatorial}
  considers rectangular objects in
  freespace or assumes a set of a hand-picked modes
  and contact points.
  In contrast, general polytopic objects in cluttered workspace
  are adopted here by allowing multiple simultaneous contacts
  without any pre-defined modes.
  \hfill $\blacksquare$
\end{remark}

\section{Proposed Solution}\label{sec:solution}
\change{As illustrated in Fig.~\ref{fig:overall},
the proposed solution tackles the above collaborative pushing problem
with an efficient keyframe-guided hybrid search algorithm
over the timed sequence of pushing modes, forces and target trajectories.}
\change{The resulting hybrid plan is then executed by
a mode switching strategy and a NMPC controller.
A re-planning module governs the execution performance
and triggers the adaptation of high-level hybrid plan as needed.}

\def\sss{\scriptscriptstyle{}}
\def\bsss{\sss{\mathrm{B}}}
\subsection{Mode Generation in Freespace}\label{subsec:loss}
This section first presents how to evaluate the feasibility of a given mode
to push the target at a desired velocity in the quasi-static condition,
which is then extended to an arc transition in the freespace.
Based on this measure, a sparse optimization algorithm is proposed to
generate a set of effective modes given the desired target motion.

\subsubsection{Quasi-static Analyses}\label{subsubsec:quasi-static}
As described in Sec.~\ref{ss:interaction_mode},
it is difficult to derive an analytical form of
the coupled dynamics in~\eqref{eq:multi_body}.
As also adopted in~\citet{lynch1992manipulation},
\change{
the quasi-static analyses assume that the motion of the target is sufficiently slow,
such that its acceleration is approximately zero and the inertia forces can be neglected.
Consequently, the target can be constantly treated as in static equilibrium, 
i.e.,~$\mathbf{q}_\mu+\mathbf{q}_{\xi}=\mathbf{0}$.
Under quasi-static assumption,
the friction~$\mathbf{q}_\mu =(f_x,f_y,m)\in \mathbb{R}^3$
is constrained on a limit surface which is approximated by the following ellipsoid:
\begin{equation}\label{eq:limit_surface}
    (f_x/f_{\max})^2+(f_y/f_{\max})^2+(m/m_{\max})^2= 1,
\end{equation}}
where~$f_{\max}$ and~$m_{\max}$ are the maximum ground friction
and moment, defined at the center of the target.
Then, the relation between~$\mathbf{q}_{\mu}$
and~$\mathbf{p}$ is given by
$(f_x,f_y,m)=-\kappa (v_x, v_y,c^2\omega)$,
where $c \triangleq m_{\max}/f_{\max}$ and $\kappa>0$.
Consequently,
given the desired velocity~$\mathbf{p}^\star$ of the target,
the generalized friction force~$\mathbf{q}_\mu$
is computed as:
\begin{equation}\label{eq:friction_force}
  \mathbf{q}_\mu \triangleq \widetilde{\eta}(\mathbf{p}^\star)=
  -\|\mathbf{D}_{1}\mathbf{D}_{2}\mathbf{p}\|^{-1}_{_2}\mathbf{D}_{2}\mathbf{p}^\star,
\end{equation}
where~$\mathbf{D}_1\triangleq \text{diag}(f^{-1}_{\sss{\max}},f^{-1}_{\sss{\max}},m^{-1}_{\sss{\max}})$,
$\mathbf{D}_2\triangleq \text{diag}(1,1,c^2)$,
and~$\widetilde{\eta}(\cdot)$ is the quasi-static approximation
of~$\eta(\cdot)$ in~\eqref{eq:target_equation}.
Thus, assuming that no slipping happens during pushing,
the coupled dynamics in~\eqref{eq:multi_body}
under control inputs~$\mathbf{u}_\mathcal{N}$ and interaction mode~$\xi$
is approximated by:
\begin{subequations}\label{eq:simplified_couple_dynamics}
  \begin{align}
    &\mathbf{q}_{\xi}=-\mathbf{q}_\mu=-\widetilde{\eta}(\mathbf{p}^\star)\in \mathrm{Q}_{\xi}; \label{eq:force_constraint}\\
    &\mathbf{v}+\bm{\omega}\times(\mathbf{c}_n-\mathbf{x})  = \mathbf{u}_n,\,  \forall n\in \mathcal{N}; \label{eq:velocity_constraint}
  \end{align}
\end{subequations}
where~\eqref{eq:force_constraint} is a condition to check if a velocity~$\mathbf{p}$
is \emph{allowed} within mode~$\xi$.
This quasi-static equilibrium can be maintained
if there exists a velocity~$\widetilde{\mathbf{p}}^\star$
that satisfies both~\eqref{eq:force_constraint}
and~\eqref{eq:velocity_constraint},
in which case the target moves at
velocity~$\widetilde{\mathbf{p}}^\star$
with each robot pushing at velocity~$\mathbf{u}_{n}$.
If not, the quasi-static condition is violated
and the desired target motion can not be maintained.

\subsubsection{Multi-directional Feasibility Estimation}
\label{subsubsec:multi-directional-feasibility-estimation}
For further clarification,
let~$\mathbb{S}$ denote the ground coordinate system
and~$\mathbb{B}$ denote the body coordinate system fixed
at the center of the target~$\mathbf{x}(t)$.
The rotation matrix from~$\mathbb{B}$ to~$\mathbb{S}$ is denoted by~$\mathbf{R}_{\mathrm{B}\rightarrow \mathrm{S}}$,
which is time-varying as the target moves.
Then, the generalized velocity of the target in~$\mathbb{B}$
is defined as~$\mathbf{p}^{\bsss}\triangleq (v_x^{\bsss},v_y^{\bsss},\omega^{\bsss})
=\mathbf{R}_{\mathrm{B}\rightarrow \mathrm{S}}^{-1}\,\mathbf{p}$.
Moreover, the coordinate~$\mathbb{B}$ provides a \emph{rotation-invariant}
description of the mode~$\xi$.
Denote by~$\mathrm{Q}_{\xi}^{\bsss}\triangleq
\{\mathbf{R}_{\mathrm{B}\rightarrow \mathrm{S}}^{-1}\mathbf{q}_{\xi},
\forall \mathbf{q}_{\xi}\in \mathrm{Q}_{\xi}\}$
the set of allowed forces~$\mathrm{Q}_{\xi}$ transformed to~$\mathbb{B}$.
It can be seen that~$\mathrm{Q}_{\xi}^{\bsss}$ is a fixed set
for a given mode~$\xi$
and independent of the target rotation~$\psi$.
Similarly, denote by~$\mathcal{P}^{\bsss}_\xi\triangleq \{\mathbf{p}^{\bsss}|-\eta(\mathbf{p}^{\bsss})\in \mathrm{Q}^{\bsss}_\xi\}$
the set of allowed velocities in mode $\xi$ in $\mathbb{B}$,
which is also fixed.

The feasibility of pushing the target
at a generalized velocity~$\mathbf{p}^{\bsss}$ in a specific mode~$\xi$
can be evaluated by solving the following constrained optimization problem:
\begin{equation}\label{eq:feasibility}
  \mathrm{J}_{\texttt{F}}(\xi,\mathbf{p}^{\bsss})\triangleq
  \min_{\mathbf{q}_{\xi}^{\bsss}\in \mathrm{Q}_{\xi}^{\bsss}}\left\|\mathbf{q}_{\xi}^{\bsss}
  +\eta(\mathbf{p}^{\bsss})\right\|_{1},
\end{equation}
where~$\mathbf{q}_{\xi}^{\bsss}$ is the combined generalized force of all robots,
and~$\eta(\mathbf{p}^{\bsss})$ is the generalized friction of the target.
Note that~$\mathbf{q}_{\xi}^{\bsss}\in \mathrm{Q}_{\xi}^{\bsss}$ is a linear constraint,
and the~$\|\cdot\|_1$ norm can be transferred to a linear objective by adding auxiliary variables.
Thus,~\eqref{eq:feasibility} can be readily solved by LP solvers,
e.g.,~\texttt{CVXOPT} from~\citet{diamond2016cvxpy}.
If~$\mathrm{J}_\texttt{F}=0$, it means that the generalized
force~$\mathbf{q}^{\bsss}_{\mathcal{\xi}}=-\eta(\mathbf{p}^{\bsss})$ is allowed in mode~$\xi$,
i.e., the target can maintain the quasi-static condition at~$\mathbf{p}^{\bsss}$,
yielding a necessary condition for feasibility.

However, due to unmodeled noises and uncertainties during the contact-rich pushing process,
the target motion may deviates from the expected trajectory.
Therefore, to facilitate online adaptation during execution,
it is {essential} to estimate the feasibility \emph{in multiple directions}.
Denote by~$\mathcal{D}$ the set of directions,
which contains the desired velocity~$\mathbf{p}^{\bsss}$ as the {main} direction,
and several orthogonal {auxiliary} directions.
Since~$\dim(\mathbf{p})=3$, five auxiliary directions are chosen to span the velocity space
with only positive weights, i.e.,
$\mathbf{p}^{\bsss}_{\sss{[1]}}\triangleq \mathbf{p}^{\bsss}, \mathbf{p}^{\bsss}_{\sss{[2]}}\triangleq \mathbf{e}_3 \times \mathbf{p}^{\bsss}_{\sss{[1]}},
\mathbf{p}^{\bsss}_{\sss{[3]}}\triangleq \mathbf{p}^{\bsss}_{\sss{[1]}}\times \mathbf{p}^{\bsss}_{\sss{[2]}},
\mathbf{p}^{\bsss}_{\sss{[4]}}\triangleq -\mathbf{p}^{\bsss}_{\sss{[1]}},
\mathbf{p}^{\bsss}_{\sss{[5]}}\triangleq -\mathbf{p}^{\bsss}_{\sss{[2]}},
\mathbf{p}^{\bsss}_{\sss{[6]}}\triangleq -\mathbf{p}^{\bsss}_{\sss{[3]}}$,
where $\mathbf{e}_3\triangleq (0,0,1)$ and~$\mathcal{D}\triangleq \{1,\cdots,6\}$.
Consequently, the multi-directional feasibility estimation is the weighted sum
of feasibility in each direction, i.e.,
\begin{equation}\label{eq:multi-directional-feasibility}
  \mathrm{J_{\texttt{MF}}}(\xi,\, \mathbf{p}^{\bsss})\triangleq \sum_{d\in \mathcal{D}}
  w_d \mathrm{J_{\texttt{F}}}(\xi,\, \mathbf{p}^{\bsss}_{\sss{[d]}}),
\end{equation}
where~the weights~$w_d>0$, $\forall d\in \mathcal{D}$ with~$w_0$ being the maximum.
As illustrated in Fig.~\ref{fig:MDF},
these two modes (all from one side or each at different sides)
have an equally-small $\mathrm{J}_{\text{F}}$  in certain directions.
But the second mode has a significantly lower $\mathrm{J}_{\text{MF}}$
in a wide range of directions \emph{around} the desired direction,
of which the balance between different directions
is tunable by the weights~$\{w_d\}$.

\begin{remark}\label{remark:auxiliary-direction}
  Note that given enough number of robots,
  a mode similar to the caging strategy in~\citet{lynch1992manipulation} is preferred
  by~\eqref{eq:multi-directional-feasibility},
  i.e.,~surrounding the target from multiple sides instead of all pushing from the same side.
  This allows the robots to potentially rectify the target motion online
  whenever it deviates from the desired velocity during execution.
  \hfill $\blacksquare$
\end{remark}

\begin{remark}\label{remark:pull}
  The above analyses still apply if the robots can \emph{pull} the target.
  The only difference is that the contact force $\mathbf{f}_n$ is no longer constrained
  by the friction cone in~\eqref{eq:friction_cone},
  which often leads to a smaller~$J_{\texttt{MF}}$ and higher feasibility.
  \hfill $\blacksquare$
\end{remark}

\begin{figure}[t]
  \centering
  \includegraphics[width=0.9\linewidth]{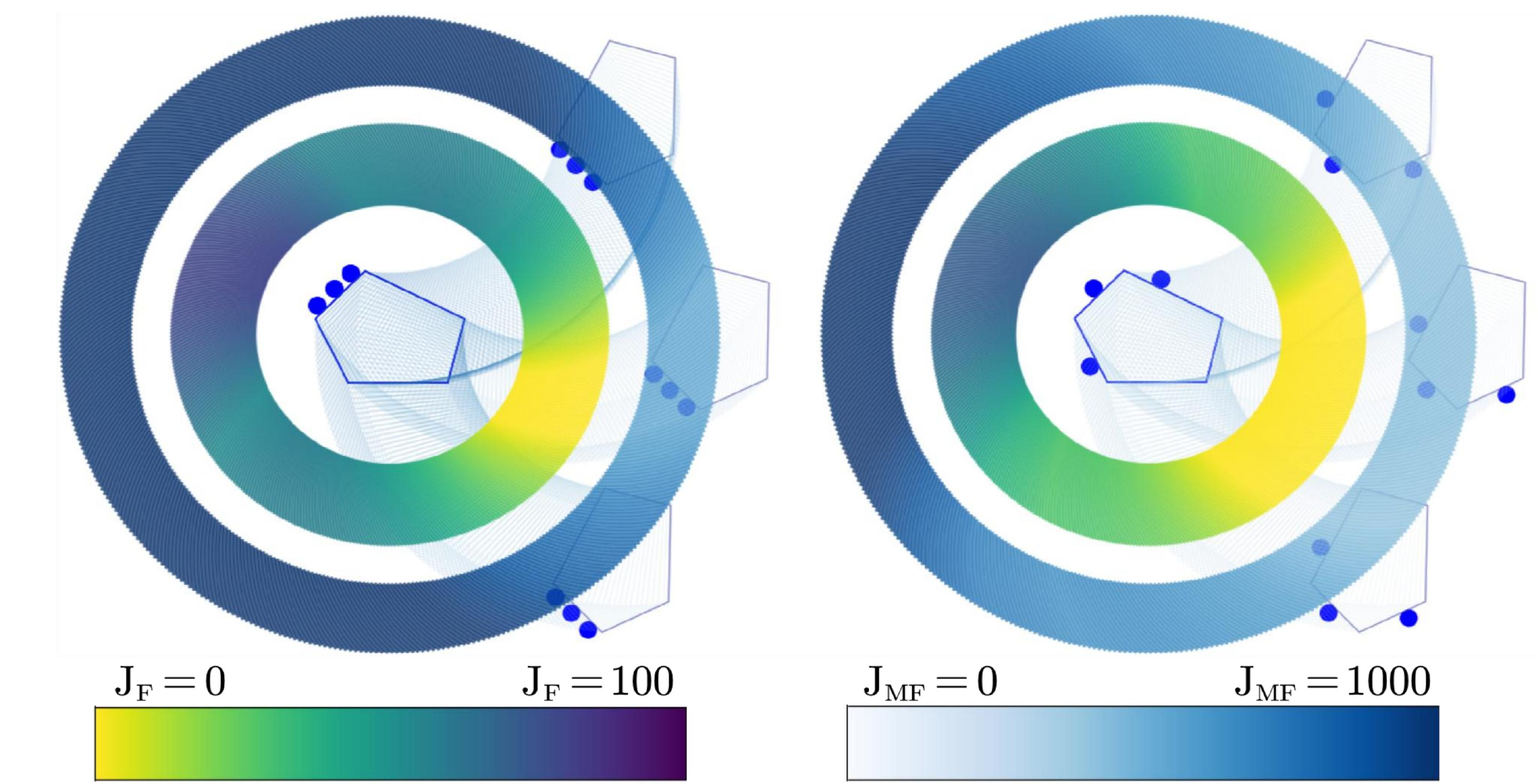}
  \caption{
    Comparison of
    $\mathrm{J}_\texttt{F}$ in~\eqref{eq:feasibility} and
    $\mathrm{J}_{\texttt{MF}}$ in~\eqref{eq:multi-directional-feasibility}
    for different directions
    between different modes (robots as blue dots and object as a white polygon).
  }
  \vspace{-4mm}
  \label{fig:MDF}
\end{figure}

\subsubsection{Loss Estimation for Arc Transition}
\label{subsubsec:arc_approximation}
Given any allowed target
velocity~$\mathbf{p}^{\bsss}= (\mathbf{v}^{\bsss},\,\omega^{\bsss})\in \mathcal{P}^{\bsss}_\xi$
in mode~$\xi$,
the resulting motion of the target is analyzed as follows.
Assuming that the robots push the target with
velocity~$\mathbf{p}(t)=\mathbf{R}_{\mathrm{B}\rightarrow \mathrm{S}}(t)\mathbf{p}^{\bsss}$
for time duration~$\bar{t}>0$,
the quasi-static condition is ensured
as~$\mathbf{p}^{\bsss} = \mathbf{R}_{\mathrm{B}\rightarrow \mathrm{S}}^{-1}(t)
\mathbf{p}(t)\in\mathcal{P}_\xi^{\bsss}$ holds.
Consequently, the change of target state within~$\mathbb{S}$
during period~$[0,\, \bar{t}]$ is given by:
\begin{equation}\label{eq:arc_trajectory}
  \begin{split}
  &\Delta \mathbf{s}\triangleq \Big{(}\int_{t=0}^{\bar{t}} \mathbf{R}_{\mathrm{B}\rightarrow \mathrm{S}}(t) dt\Big{)}\,
  \mathbf{p}^{\bsss}\\
  &=\left[ \begin{matrix}
    \lim\limits_{\Delta\psi\rightarrow \Delta \Psi}
    \mathbf{Rot}\left( \Delta\psi+\psi_0 -\frac{\pi}{2} \right)|_{0}^{\Delta\psi}\frac{1}{\Delta\psi} &		0\\
    0&		1\\
  \end{matrix} \right]  \bar{t}\ \mathbf{p}^{\bsss}\\
  &=\mathbf{\Phi}(\psi_0,\, \Delta \Psi) \ \bar{t}\ \mathbf{p}^{\bsss},
  \end{split}
\end{equation}
where~$\mathbf{Rot}: (-\pi,\pi] \rightarrow \mathbb{R}^{2\times 2}$
is the standard rotation matrix of a given angle;
$\Delta \Psi=\omega^{\bsss}\bar{t}$ is the change of target orientation.
Note that the limit operation above is employed
due to the singularity at~$\omega^{\bsss}=0$,
in which case the limit is given by~$\mathbf{Rot}(\psi_0)$,
yielding a pure translational
motion~$\Delta \mathbf{s}=\mathbf{R}_{\mathrm{B}\rightarrow \mathrm{S}}(0)\mathbf{p}^{\bsss}\bar{t}$.
Thus, the final state of the target is given
by~$\mathbf{s}_{\bar{t}}\triangleq \mathbf{s}_0+\Delta \mathbf{s}$
and its trajectory by
$\mathbf{Trj}(\mathbf{s}_0,\mathbf{p}^{\bsss},\bar{t})$.

\begin{lemma}\label{lm:velocity_to_arc}
Given a fixed velocity~$\mathbf{p}^{\bsss}$, the resulting
trajectory~$\mathbf{Trj}(\mathbf{s}_0,\mathbf{p}^{\bsss},\bar{t})$ is an arc
with radius~$\rho=\|\mathbf{v}^{\bsss}\|/|\omega^{\bsss}|$ and angle~$\Delta \psi$;
or a straight line with length~$\|\mathbf{v}^{\bsss}\|\bar{t}$.
\end{lemma}
\begin{proof}
(Sketch) If $\Delta \Psi\neq 0$,
consider the point $\mathbf{x}_c\triangleq \mathbf{x}_0-\mathbf{Rot}(\psi_0-\pi/2)\mathbf{v}^{\bsss}/\omega^{\bsss}$.
For any~$t\in[0,\bar{t}]$,
the target position by~\eqref{eq:arc_trajectory} is given by
$\mathbf{x}_t\triangleq \mathbf{x}_0+\Delta \mathbf{x}$.
It can be verified that $\|\mathbf{x}_t-\mathbf{x}_c\|_2$ is a constant
and equals to $\|\mathbf{v}^{\bsss}\|/|\omega^{\bsss}|$.
The rest of the proof follows similar arguments as for~\eqref{eq:arc_trajectory}.
\end{proof}

\begin{figure}[t]
  \centering
  \includegraphics[width=0.96\linewidth]{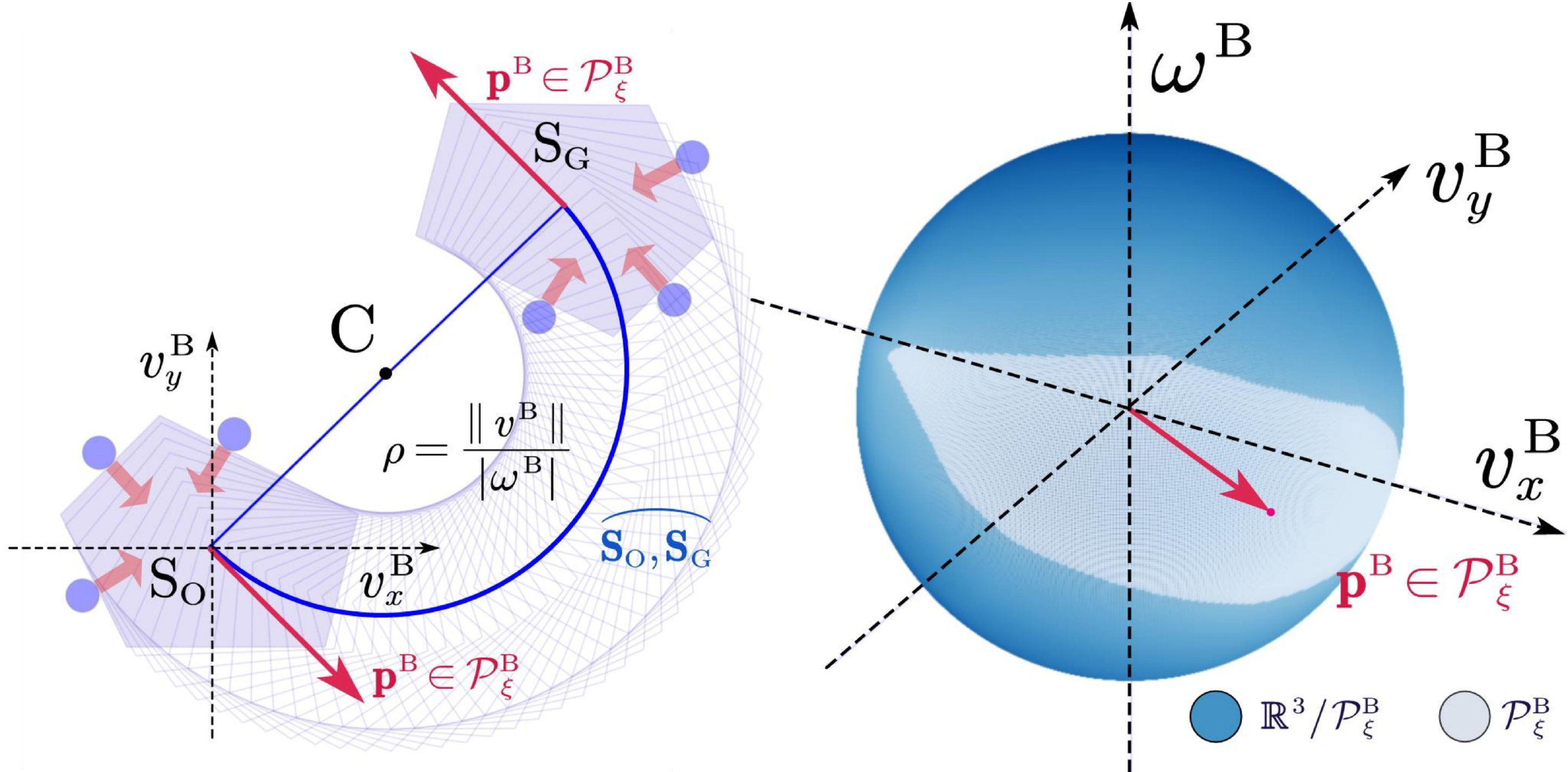}
  \caption{Arc transitions~(\textbf{Left})
  and the set of allowed velocities~$\mathcal{P}_{\xi}^{\bsss}$ (\textbf{Right})
  visualized on a unit sphere.}
  \vspace{-3mm}
\label{fig:ARC}
  \vspace{-3mm}
\end{figure}

\begin{lemma}\label{lm:arc_existence}
  \change{Any given pair of initial state~$\mathbf{s}_0$ and goal state~$\mathbf{s}_{\texttt{G}}$
  can be connected by a unique arc trajectory $\mathbf{Trj}(\mathbf{s}_0,\mathbf{p}^{\bsss},\bar{t})$
  such that $\omega^{\bsss}\,\bar{t}\in \left[-\pi,\pi \right)$.}
\end{lemma}
\begin{proof}
  \change{
  (Sketch)
  Let $\Delta\Psi$ in~\eqref{eq:arc_trajectory} be $\psi_{\texttt{G}}-\psi_{0}$ for the arc transition.
  Since the matrix $\mathbf{\Phi}(\psi_0,\Delta \Psi)$
  is invertible,
  it implies that~$\mathbf{p}^{\bsss}\bar{t}=
  \mathbf{\Phi}^{-1}(\psi_0,\Delta \Psi)
  (\mathbf{s}_{\texttt{G}}-\mathbf{s}_{0})$ is uniquely determined.
  Moreover, scaling the velocity~$\mathbf{p}^{\bsss}$ by a factor~$k\in\mathbb{R}^+$
  leads to the same arc trajectory with duration~$\bar{t}/k$.
  Thus, the trajectory~$\mathbf{Trj}(\mathbf{s}_0,\mathbf{p}^{\bsss},\bar{t})$
  is uniquely determined.}
\end{proof}

With Lemma~\ref{lm:arc_existence},
we denote by~$\overgroup{\mathbf{s}_0\mathbf{s}_{\texttt{G}}}\triangleq \mathbf{Trj}(\mathbf{s}_0,\mathbf{p}^{\bsss},\bar{t})$
the unique arc transition from~$\mathbf{s}_0$ to~$\mathbf{s}_{\texttt{G}}$,
of which the arc length is given by
$\|\overgroup{\mathbf{s}_0\mathbf{s}_{\texttt{G}}}\|\triangleq\|\mathbf{p}^{\bsss}\bar{t}\|_2$.
Consequently, a loss estimation
is proposed for choosing a mode~$\xi$
under an arc transition~$\overgroup{\mathbf{s}_0\mathbf{s}_{\texttt{G}}}$:
\begin{equation}\label{eq:loss_estimation}
  \mathrm{J_{\texttt{MF}}}(\xi,\, \overgroup{\mathbf{s}_0\mathbf{s}_{\texttt{G}}})
  =\mathrm{J_{\texttt{MF}}}(\xi,\, \mathbf{p}^{\bsss})\,\|\mathbf{p}^{\bsss}\bar{t}\|_2,
\end{equation}
where $\mathrm{J_{\texttt{MF}}}(\cdot)$ is
the multi-directional feasibility estimation in~\eqref{eq:multi-directional-feasibility}.
In other words, the loss estimation for choosing a mode under an arc transition
is estimated by the feasibility of the associated velocity
but weighted by the arc length.

\begin{figure}[t]
  \centering
  \includegraphics[width=1.0\linewidth]{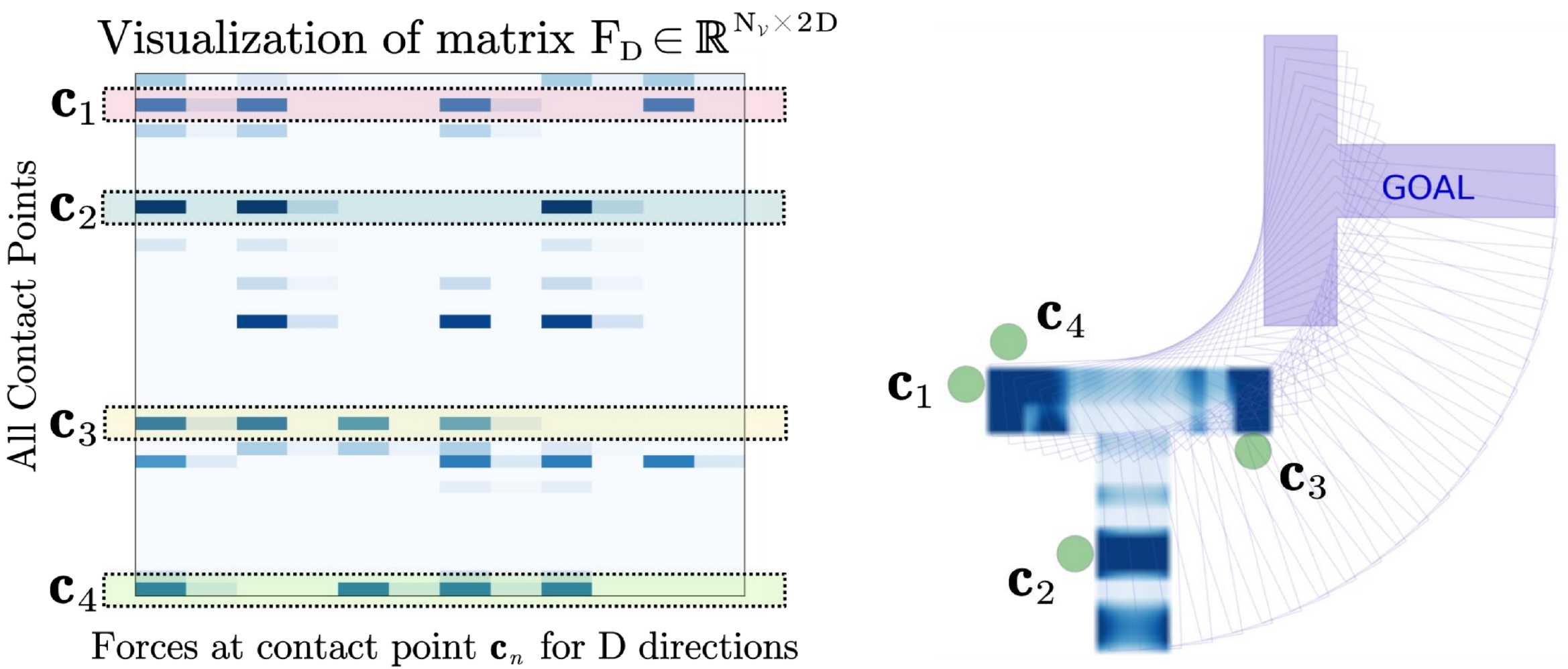}
  \caption{Illustration of the mode generation process via sparse
  optimization in~\eqref{eq:sparse-opt}.
  \textbf{Left}: the optimized matrix~$\mathbf{F}^\star_{\mathbf{D}}$ of dimention $N_{\mathcal{V}}\times 2D$,
    where the magnitude of elements is represented by the color intensity;
  \textbf{Right}: the best mode selected for the given arc transition
  within $\mathbf{\Xi}(\protect \overgroup{\mathbf{s}_0\mathbf{s}_{\bar{t}}})$.}
  \label{fig:MDF-SO}
  \vspace{-6mm}
\end{figure}
\subsubsection{Mode Generation for Arc Transition}\label{subsec:mode_generation}

With the proposed loss function in~\eqref{eq:loss_estimation},
the {optimal} mode~$\xi^\star$
for a desired arc transition~$\overgroup{\mathbf{s}_0\mathbf{s}_{\texttt{G}}}$
can be determined.
However, as described earlier, the set of all possible modes~$\Xi$
in this work is infinite,
of which the parameter space grows exponentially with the number of robots
and the boundary length of the target.
To tackle this issue,
a fast and scalable method to generate suitable modes for a given arc transition is proposed,
called~\emph{mode generation via sparse optimization} (MG-SO).
To begin with,
each side~$(p_v,p_{v+1})$ of the target is divided into~$N_{v}>0$ segments,
$\forall v\in\{1,\cdots,V\}$.
The center of each segment is denoted by~$p_{v,i}$, $\forall i\in\{1,\cdots,N_{v}\}$.
Then, consider a virtual mode~$\xi^\star$ that
employs~$\{p_{v,i},\forall i, \forall v\}$ as the contact points,
i.e., $\xi^\star \triangleq p_{\sss{1,1}}\cdots p_{\sss{1,N_1}}\cdots p_{\sss{V,1}}
\cdots p_{\sss{V,N_V}}$.
Note that the total number of contact points is~$N_\mathcal{V}=\sum_{v=1}^V N_v$,
which is much larger than the number of robots~$N$.
The associated force~$\mathbf{F}_{\xi^\star}\triangleq (\mathbf{F}_{\texttt{n}},\, \mathbf{F}_{\texttt{t}})$
follows the definition in~\eqref{eq:F-xi},
among which the forces along the multiple directions~$\mathbf{p}^{\bsss}_{\sss{[d]}}$
in~\eqref{eq:multi-directional-feasibility}
are given by~$\mathbf{F}_{\xi^\star}^{\sss{[d]}}\triangleq (\mathbf{F}^{\sss{[d]}}_{\texttt{n}},\,
\mathbf{F}^{\sss{[d]}}_{\texttt{t}})$, $\forall d\in \mathcal{D}$.
Subsequently, the following sparse optimization problem is formulated:
\begin{subequations}\label{eq:sparse-opt}
  \begin{align}
    \underset{{\mathbf{F}_{\mathbf{D}}}}{\textbf{min}}
    &\sum_{n\in \mathcal{N}_{\mathcal{V}}}
  (\|\mathbf{F}^n_{\mathbf{D}}\|_{\infty}+\|\mathbf{F}^n_{\mathbf{D}}\|_{\sss{1}})
    + \sum_{d \in \mathcal{D}}  \|\mathbf{q}^{\bsss}_{[d]}+\eta(\mathbf{p}^{\bsss}_{\sss{[d]}})\|_{\sss{1}} \label{eq:multi-feasibility} \\
    &\textbf{s.t.}\quad
  \mathbf{q}^{\bsss}_{[d]}=\mathbf{J} \mathbf{F}_{\xi^\star}^{\sss{[d]}}, \ \mathbf{F}_{\xi^\star}^{\sss{[d]}}\in\mathcal{F}_{\xi^\star},\;
    \forall d\in \mathcal{D}; \label{eq:friction_cone}
  \end{align}
\end{subequations}
where the decision variables~$\mathbf{F}_{\mathbf{D}}\triangleq
\mathbf{F}^{\sss{[1]}}_{\texttt{n}} \mathbf{F}^{\sss{[1]}}_{\texttt{t}}\cdots
  \mathbf{F}^{\sss{[D]}}_{\texttt{n}} \mathbf{F}^{\sss{[D]}}_{\texttt{t}}\in
  \mathbb{R}^{\sss{N}_{\mathcal{V}}\times\sss{2D}}$;
$\mathbf{F}^n_{\mathbf{D}}\in \mathbb{R}^{\sss{2D}}$ is the $n$-th row of $\mathbf{F}_{\mathbf{D}}$;
and the first term in~\eqref{eq:multi-feasibility} contains
both the $\mathcal{L}_1$ norm and $\mathcal{L}_\infty$ norm of~$\mathbf{F}^n_{\mathbf{D}}$,
while the second term is the feasibility estimation in~\eqref{eq:feasibility} for each direction;
the constraint of friction cone in~\eqref{eq:friction_cone} is defined in~\eqref{eq:force-limit}.
Note the first term in~\eqref{eq:multi-feasibility}
is designed to improve the {sparsity} of matrix~$\mathbf{F}_{\mathbf{D}}$, i.e.,
$\sum_{n\in \mathcal{N}_{\mathcal{V}}}\|\mathbf{F}^n_{\mathbf{D}}\|_{1}$
penalizes the number of non-zero elements in $\mathbf{F}_{\mathbf{D}}$,
while $\sum_{n\in \mathcal{N}_{\mathcal{V}}} \|\mathbf{F}^n_{\mathbf{D}}\|_{\infty}$ penalizes
the number of non-zero rows in $\mathbf{F}_{\mathbf{D}}$,
i.e., the \emph{row sparsity} in~\citet{leon2006linear}.
This is crucial as a row of zeros indicates that forces at the corresponding
contact point are not required.
In other words,  only key contact points are associated with
the non-zero rows of~$\mathbf{F}_{\mathbf{D}}$.

The above optimization problem is again a linear program (LP),
which can be solved efficiently with a LP solver.
Denote by~$\mathbf{F}^\star_{\mathbf{D}}$ the optimal solution,
which might contain more than~$N$ non-zero rows.
Thus, a set of modes for this arc transition,
denoted by~$\mathbf{\Xi}(\overgroup{\mathbf{s}_0\mathbf{s}_{\texttt{G}}})$,
can be generated as follows.
All rows of~$\mathbf{F}^\star_{\mathbf{D}}$ are sorted in descending order
by the value of $\|\mathbf{F}^n_{\mathbf{D}}\|_{\infty}+\|\mathbf{F}^n_{\mathbf{D}}\|_{1}$.
Then, the first mode is given by the collecting the first~$N$ rows,
and the subsequent mode is generated by combing the first~$N-1$ rows with a random row.
This process is repeated until a certain size
of~$\mathbf{\Xi}(\overgroup{\mathbf{s}_0\mathbf{s}_{\texttt{G}}})$ is reached.
An example of the results is shown in Fig.~\ref{fig:MDF-SO}.
It can be seen that the sparsity optimization is effective as
most elements of the matrix~$\mathbf{F}^\star_{\mathbf{D}}$ are zero,
while the non-zero elements are concentrated in few rows,
as visualized by the color intensity.

\begin{remark}\label{remark:mg-so}
  Due to the polynomial complexity of solving linear programs,
  the above mode generation method is scalable to
  arbitrary number of robots and any polytopic target shape.
  \change{For non-polytopic shapes, its exterior contour should be discretized
first to obtain the set of potential contact points,
which are then filtered by whether there exists enough room for a robot.
Given these contact points, the same optimization in~\eqref{eq:sparse-opt}
can be adopted to generate modes for an arc transition.
  More case studies of are provided in Sec.~\ref{sec:experiments}.}
  \hfill $\blacksquare$
\end{remark}

\begin{figure}
  \centering
  \includegraphics[width=1.0\linewidth]{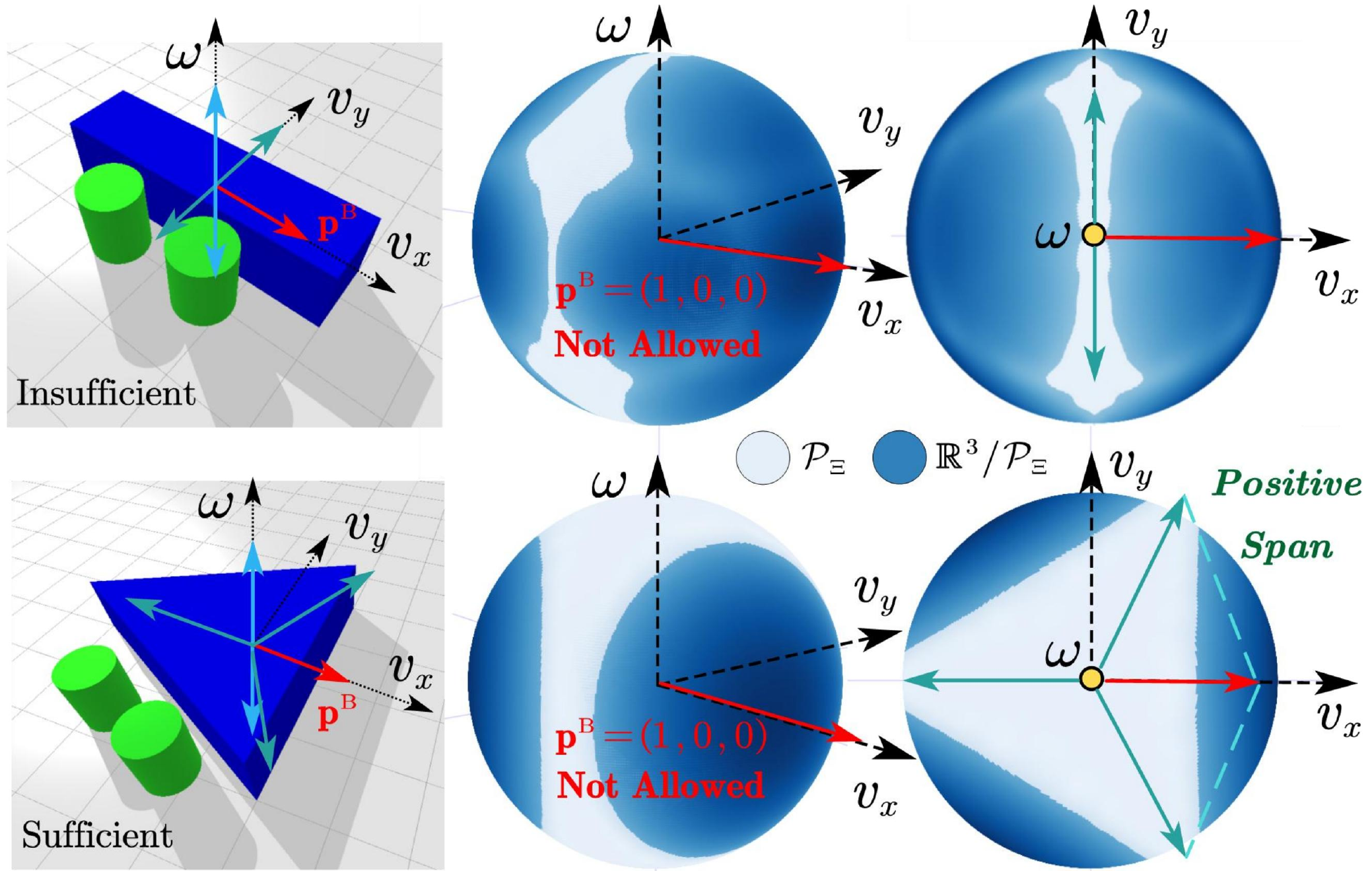}
  \vspace{-3mm}
  \caption{Given all modes~$\Xi$ for two robots
  and different objects (\textbf{Left}),
  the set of allowed velocities $\mathcal{P}_{\Xi}$ is shown (in white)
  on a unit sphere of $v_x-v_y-\omega$ velocities (\textbf{Middle}) with
  the top-down projection (\textbf{Right}) from the angular velocity~$\omega$.}
  \label{fig:sufficient}
  \vspace{-5mm}
\end{figure}

\def\sss{\scriptscriptstyle{}}
\def\bsss{\sss{\mathrm{B}}}
\subsection{Sufficient Modes for General Trajectory}
\label{subsubsec:trajectory_approximation}
The previous section provides an efficient method
to generate a set of modes for an arc transition in the freespace,
i.e., via a constant velocity $\mathbf{p}^{\bsss}$ in $\mathbb{B}$.
However, this velocity may not be allowed within~$\Xi$,
i.e., $\mathbf{p}^{\bsss}\notin \bigcup_{\xi\in\Xi}\mathcal{P}^{\bsss}_\xi$,
or this arc might intersect with obstacles.
Consequently, given a general collision-free trajectory $\overline{\mathbf{S}}$,
the main idea is to approximate it
by a sequence of \emph{allowed} arc transitions, i.e.,
\begin{equation}\label{eq:approximate-arc}
\widetilde{\mathbf{S}}\triangleq \bigcup_{\ell\in\mathcal{L}}\overgroup{\mathbf{s}_{\ell-1}\mathbf{s}_{\ell}},
\end{equation}
where~$\mathcal{L}\triangleq \{1,\cdots,L\}$ and $L>0$ is the number of arcs.
The approximation error is measured by the Hausdorff
distance~$\mathcal{H}(\overline{\mathbf{S}},\, \widetilde{\mathbf{S}})$
from~\citet{rockafellar2009variational}.

To begin with, when any generalized velocity~$\mathbf{p}$ is allowed in
at least one mode within $\Xi$,
i.e., $\mathbf{p}\in \bigcup_{\xi\in\Xi} \mathcal{P}_\xi$
holds for all~$\mathbf{p}\in\mathbb{R}^3$,
the above approximation can be achieved by dividing $\overline{\mathbf{S}}$ into sufficiently small segments,
each of which can be approximated by an arc.
However, this condition is not easily met,
e.g., when the number of robots is small,
or certain points on the edge can not be contact points,
or the combined forces are not sufficient in certain directions.
As shown in the Fig.~\ref{fig:sufficient},
when the maximum pushing force of each robot is only
half of the maximum friction force,
the above condition can not be satisfied.
More importantly, verifying this condition requires traversing
all velocities,
and calling MG-SO for each direction to verify
whether it is allowed in at least one mode,
which is computationally expensive.
Thus, a weaker and easier-to-check notion of sufficiency is proposed as follows:
\begin{definition}\label{def:sufficient}
The set of all possible modes~$\Xi$ is called \emph{sufficient}
if two conditions hold:
(i) any generalized velocity~$\mathbf{p}$ can be positively decomposed
into several key velocities,
i.e., $\mathbf{p}=\sum_{i\in \mathcal{I}} \lambda_i \mathbf{p}_i$,
where $\lambda_i\geq 0$;
(ii) $\mathbf{p}_i$ is allowed in at least one mode $\xi\in\Xi$, $\forall i\in \mathcal{I}$,
i.e., $\mathbf{p}_i\in \mathcal{P}_{\Xi}\triangleq\bigcup_{\xi\in\Xi} \mathcal{P}_\xi$.
\hfill $\blacksquare$
\end{definition}

In other words,
only the key velocities need be verified as allowed in at least one mode,
rather than all possible velocities.
For instance,
the six directions~$\{\mathbf{p}_{[d]}^{\bsss}\}$
in~\eqref{eq:multi-directional-feasibility} can be chosen as the key velocities,
for which the above conditions holds by definition.
Furthermore, it is shown below that a set of sufficient modes can generate
arc transitions to approximate any general trajectory of the target
with arbitrary accuracy.
As shown in Fig.~\ref{fig:sufficient},
the allowed velocities $\mathcal{P}_{\Xi}$ for all possible modes~$\Xi$
of two robots and a triangular object has a triangular distribution,
thus easier to span all directions.
In contrast for rectangular objects,
the allowed velocities are concentrated on $v_y-\omega$ plane,
i.e., difficult to span in the $v_x$ direction.

\begin{figure}
  \centering
  \includegraphics[width=1.0\linewidth]{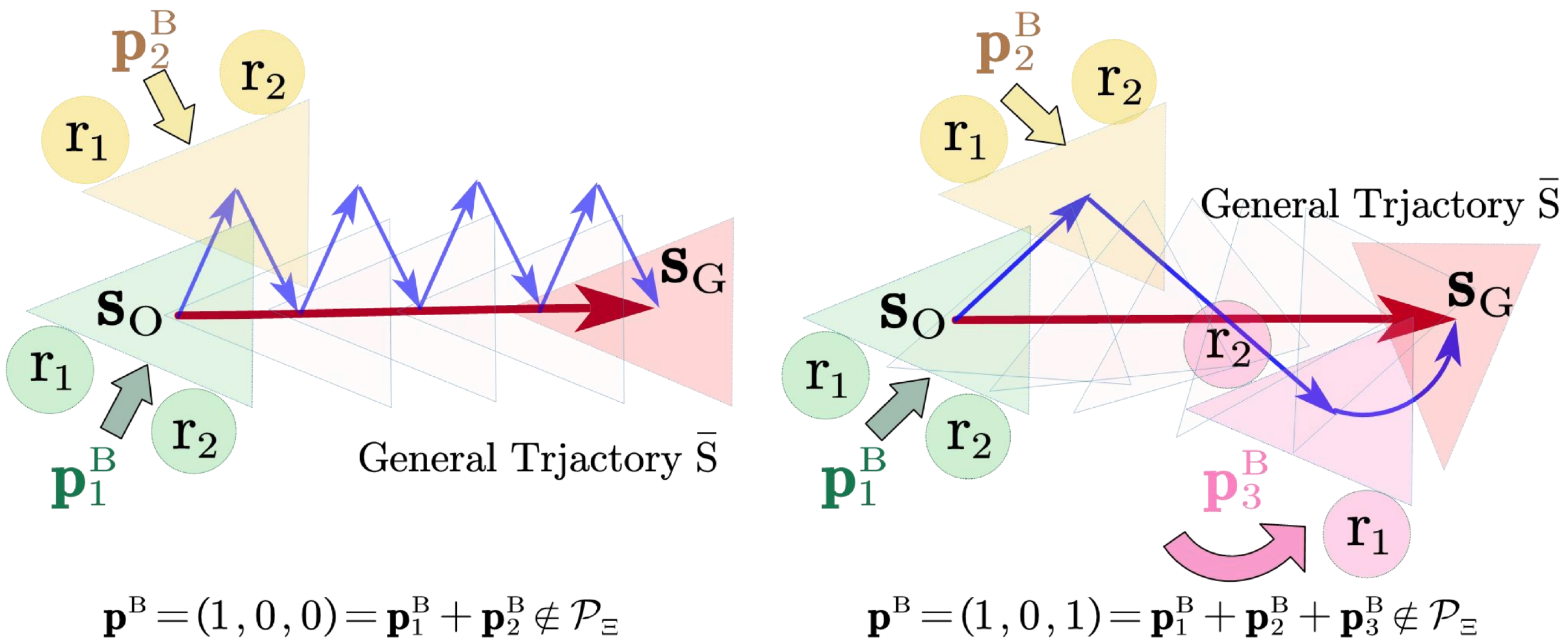}
  \vspace{-4mm}
  \caption{Trajectory approximation via a set of sufficient modes.}
  \label{fig:trajectory_approximation}
  \vspace{-3mm}
\end{figure}

\begin{theorem}\label{thm:feasibility}
  If the set of all possible modes~$\Xi$ is sufficient,
  then for any collision-free trajectory $\overline{\mathbf{S}}$
  and any error bound $e>0$, there exists a sequence of allowed arc transitions
  $\widetilde{\mathbf{S}}$ by~\eqref{eq:approximate-arc} such that
  $\mathcal{H}(\overline{\mathbf{S}},\,\widetilde{\mathbf{S}})\leq e$.
\end{theorem}
\begin{proof}
  First, $\overline{\mathbf{S}}$ can be divided into $L$ segments evenly by length,
  such that the $\ell$-th segment $(\mathbf{s}_{\ell-1},\mathbf{s}_{\ell})$
  is denoted as $\overline{\mathbf{S}}_{\ell}$.
  Assuming that $\overline{\mathbf{S}}_{\ell-1}$ is already approximated by
  an arc that ends at~$\mathbf{s}^+_{\ell-1}$.
  Then, the desired state change~$\Delta \mathbf{s}_{\ell}=\mathbf{s}_{\ell}-\mathbf{s}^+_{\ell-1}$
  can be approximated by the arc
  associated with velocity~$\mathbf{p}^{\bsss}_{\ell}$ and duration $T_{\ell}$,
  i.e., $\Delta \mathbf{s}_{\ell}=
  \mathbf{\Phi}(\psi_{\ell-1},\omega^{\bsss}_{\ell} T_{\ell})\,\mathbf{p}^{\bsss}_{\ell}\,T_{\ell}$.
  Second, since~$\Xi$ is sufficient,
  it holds that~$\mathbf{p}_{\ell}^{\bsss}=\sum_{i\in \mathcal{I}} \lambda_i \mathbf{p}_i$,
  where $\lambda_i \geq 0$ and $\mathbf{p}_i\in \mathcal{P}_{\Xi}$.
  Since $\mathrm{\dim}(\mathcal{P})=3$,
  there exists $\lambda_j^\star\geq 0$ and $\mathbf{p}_1^\star,\mathbf{p}_2^\star,\mathbf{p}_3^\star\in \mathcal{P}_{\sss{\Xi}}$
  such that $\mathbf{p}^{\bsss}_\ell=\sum_{j=1}^3 \lambda_{j}^\star\mathbf{p}^\star_j$.
  Consider the sequence of arcs $(\mathbf{p}^\star_1,t_1)(\mathbf{p}^\star_2,t_2)(\mathbf{p}^\star_3,t_3)$
  that starts from $\mathbf{s}_{\ell-1}^+$,
  where $t_j=\lambda_j^\star t$ is the duration of velocity $\mathbf{p}^\star_j$.
  The resulting change of state
  and its derivative at $t=0$ 
  is given by:
  \begin{equation}
      \Delta\mathbf{s}^{\texttt{seq}}_{\texttt{3}}\triangleq
      \sum_{j=1}^3\mathbf{\Phi}(\psi_{j-1},\omega_j t_j) \mathbf{p}^\star_j\, t_j;
      \ \frac{\mathrm{d}\mathbf{s}^{\texttt{seq}}_{\texttt{3}}}{\mathrm{d} t}=
      \sum_{j=1}^{3} \lambda_j \mathbf{p}^\star_j,
  \end{equation}
  where $\psi_{j}\triangleq \psi_0+\sum_{k=1}^j \omega_k t_k$.
  Denote by $\widetilde{\mathbf{S}}_{\ell}$ the trajectory from this $3$-step motion,
  for the $\ell$-th segment~$\overline{\mathbf{S}}_{\ell}$.
  Their Hausdorff distance is given by:
  $ h_{\ell}
    \leq (\max_{j,t}\big\{|\mathbf{p}_j^{\bsss}-\frac{\mathrm{d}\overline{\mathbf{S}}(t)}{\mathrm{d}t}|\big\})
    T_{\ell}+\epsilon_{\ell-1}
    \leq \mathcal{O}(\tau)+\epsilon_{\ell-1}$,
  where $\tau\triangleq \frac{1}{L}$ and
  $\epsilon_{\ell-1}\triangleq \|\mathbf{s}^+_{\ell-1}-\mathbf{s}_{\ell-1}\|_2$.
  The difference in the final state between~$\overline{\mathbf{S}}_{\ell}$
  and $\widetilde{\mathbf{S}}_{\ell}$ is given by
  $\epsilon_{\ell}=\Delta \mathbf{s}_{\ell} -\Delta \mathbf{s}_{\texttt{3}}^{\texttt{seq}}
  = \mathbf{p}^{\bsss} T_{\ell}-\mathbf{p}^{\bsss} T_{\ell}+\mathcal{O}(T_{\ell}^2)
  =\mathcal{O}(\tau^2)$.
  Moreover, via the same approximation for~$\overline{\mathbf{S}}_{\ell-1}$,
  the error $\epsilon_{\ell-1}$ is also bounded by $\mathcal{O}(\tau^2)$,
  meaning that the distance~$h_{\ell}\leq \mathcal{O}(\tau)+\mathcal{O}(\tau^2)=\mathcal{O}(\tau)$.
  Thus, the Hausdorff distance between $\overline{\mathbf{S}}$
  and $\widetilde{\mathbf{S}}$ is bounded by
  $
  \mathcal{H}(\overline{\mathbf{S}},\,
  \bigcup_{l\in\mathcal{L}}\overgroup{\mathbf{s}_{\ell-1}\mathbf{s}_{\ell}})
  \leq \max_{\ell\in\mathcal{L}}\{h_{\ell}\}
  \leq \mathcal{O}(\tau)$.
  Thus, for any error bound~$e>0$, there exits a sufficiently large~$L$
  such that $\mathcal{H}(\overline{\mathbf{S}},\,\widetilde{\mathbf{S}})<e$ holds.
\end{proof}

\begin{remark}
  As shown in Fig.~\ref{fig:trajectory_approximation},
  the proof above also provides a general method
  for approximating a collision-free trajectory
  by a sequence of allowed arcs.
  However, due to the uniform discretization,
  it requires a large number of arcs for a high accuracy,
  yielding over-frequent mode switches.
  Thus, a more efficient hybrid search algorithm is proposed
  in the sequel. \hfill $\blacksquare$
\end{remark}
\begin{remark}
  Note that the sufficiency in Def.~\ref{def:sufficient}
  is defined for \emph{all} possible modes~$\Xi$,
  which is determined by the multi-pusher system
  and different from the generated modes
  $\mathbf{\Xi}(\overgroup{\mathbf{s}_0\mathbf{s}_{\texttt{G}}})$
  for an arc transition via~\eqref{eq:sparse-opt}.
  However, it can be verified whether these modes are sufficient
  for a given trajectory. \hfill $\blacksquare$
\end{remark}


\subsection{Hierarchical Hybrid Search}\label{subsec:hybrid}
Given the above mode generation method,
the section presents the hierarchical hybrid search algorithm
to generate the hybrid plan as the timed sequence of modes
and target states.
In particular, it consists of two main steps:
(i) a collision-free guiding path is generated for the target given the initial and goal states;
(ii) a hybrid search algorithm that iteratively decomposes the guiding path into arc segments
and selects the optimal modes,
to minimize a balanced cost of the resulting hybrid plan.

\subsubsection{Collision-free Guiding Path}\label{subsec:guiding-path}
A collision-free guiding path for the target
from~$\mathbf{s}_{\texttt{O}}$ to~$\mathbf{s}_{\texttt{G}}$
is determined first as a sequence of states
within its statespace~$\mathcal{S}$,
such that its polygon boundary does not intersect with any obstacle at all time.
The~$\texttt{A}^\star$ algorithm as in~\citet{lavalle2006planning} is adopted with
the distance-to-goal being the search heuristic.
The cost function for edges in the search
graph~$C_{\texttt{A}^\star}:\mathcal{S} \times \mathcal{S}\rightarrow \mathbb{R}^+$
is modified to incorporate the feasibility
of generating such motion via collaborative pushing, i.e.,
\begin{equation}\label{eq:astar-heuristic}
  C_{\texttt{A}^\star}(\mathbf{s}_1,\,\mathbf{s}_2)\triangleq  \|\mathbf{s}_1-\mathbf{s}_2\|_2
  +\mathrm{J_{\texttt{MF}}}(\xi^\star,\, \overgroup{\mathbf{s}_1\mathbf{s}_2}),
\end{equation}
where~$(\mathbf{s}_1,\mathbf{s}_2)$ is any edge in the search graph;
$\mathrm{J_{\texttt{MF}}}$ is the loss estimation from~\eqref{eq:multi-directional-feasibility},
within which~$\xi^\star$ is the optimal interaction mode
for the arc transition~$\overgroup{\mathbf{s}_1 \mathbf{s}_2}$.
Note that given a fixed time step, the minimum loss~$\mathrm{J}_{\texttt{MF}}$
for different arcs within the local coordinate
can be pre-calculated and stored for online evaluation.
Denote by~$\overline{\mathbf{S}}\triangleq \mathbf{s}_0 \cdots\mathbf{s}_{L}$
the resulting path.

\subsubsection{Keyframe-guided Hybrid Search}\label{subsec:hybrid-search}
Given the guiding path of the target, a hybrid plan
is determined as a sequence of trajectory segment and the associated interaction mode
(including contact points and forces),
such that the target can be pushed along the guiding path.
\change{As shown in Fig.~\ref{fig:kghs}, a hybrid search algorithm called
{keyframe-guided hybrid search} (KG-HS) is presented in this section.
It iteratively decomposes the guiding path~$\overline{\mathbf{S}}$
into segments of arcs between keyframes,
such that each segment can be tracked by single pushing mode.}

Specifically, the search space is defined as~$\mathcal{V}\triangleq \{\nu\}$,
where each node~$\nu$ is a hybrid plan defined as a sequence of \emph{keyframes},
i.e.,~$\nu \triangleq \kappa_0 \cdots \kappa_{\ell} \cdots \kappa_{L_{\nu}}$,
with~$\kappa_{\ell}\triangleq (\mathbf{s}_{\ell},\xi_{\ell})$
being the $\ell$-th keyframe, $\forall \ell = 1,\cdots,L_{\nu}$.
Note that~$\mathbf{s}_{\ell}\in \overline{\mathbf{S}}$
is the state of the target such
that~$\mathbf{s}_0=\mathbf{s}_{\texttt{O}}$
and~$\mathbf{s}_{L_{\nu}}=\mathbf{s}_{\texttt{G}}$,
while~$\xi_{\ell}$ is the mode for the
segment~$\overgroup{\mathbf{s}_{l}\mathbf{s}_{l+1}}$.
Furthermore, the cost of a hybrid plan is estimated as follows:
\begin{equation}\label{eq:cost-plan}
  \mathrm{C}(\nu) \triangleq \sum_{l=0}^{L_{\nu}-1} \mathrm{J_{\texttt{MF}}}
  (\xi_{\ell},\,\overgroup{\mathbf{s}_{l}\mathbf{s}_{l+1}})
                  +w_{\texttt{t}}\mathrm{T}_{\texttt{sw}}(\xi_\ell,\,\xi_{\ell+1}),
\end{equation}
where the loss~$\mathrm{J_{\texttt{MF}}}(\cdot)$ is defined in~\eqref{eq:loss_estimation};
$w_{\texttt{t}}>0$ is a design parameter;
and~$\mathrm{T}_{\texttt{sw}}(\cdot)>0$ estimates the time required
for the robots to switch from mode~$\xi_{\ell}$ to $\xi_{\ell+1}$,
which is described later.
Note that the objective function~$J(\cdot)$ in the
original problem of~\eqref{eq:problem} is the
summed cost of a hybrid plan~$\nu$ by~\eqref{eq:cost-plan}.

\begin{figure}[t!]
  \centering
  \includegraphics[width=1\linewidth]{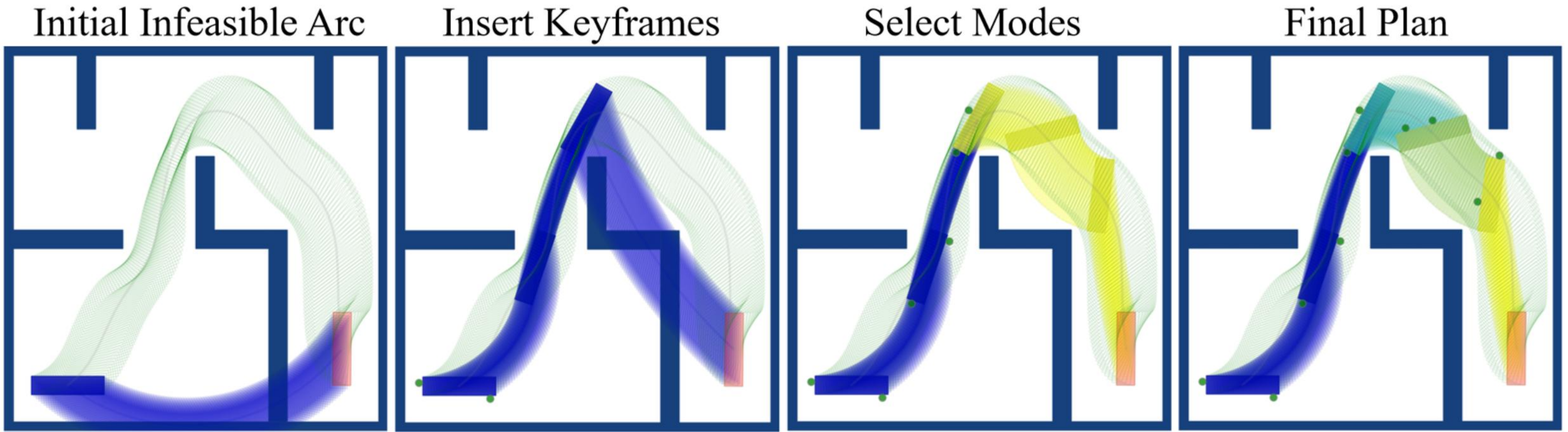}
  \caption{\change{Illustration of the KG-HS algorithm in Sec.~\ref{subsec:hybrid}.}}
  \label{fig:kghs}
  \vspace{-4mm}
\end{figure}

First, a priority queue is used to store the nodes to be visited,
which is initialized as~$Q=\{\nu_0\}$,
where~$\nu_{\texttt{O}}\triangleq (\mathbf{s}_{\texttt{O}},\mathbf{\varnothing})(\mathbf{s}_{\texttt{G}},
\mathbf{\varnothing})$
and~$\mathbf{\varnothing}$ indicates that no mode has been assigned to the arc.
Then, the search space~$\mathcal{V}$ is explored via an iterative procedure of node selection and expansion,
as summarized in Alg.~\eqref{alg:KGHS}.

(I) \textbf{Selection}.
The node with the lowest estimated cost within $Q$ is popped,
i.e.,~$\nu^\star \triangleq \textbf{argmin}_{\nu \in \mathcal{V}} \{{\mathrm{C}(\nu)+\mathrm{H}(\nu)}\}$,
where~$\mathrm{H}(\cdot)$ is an under-estimation of the cost for the segments in~$\nu$
that are not assigned, i.e.,
$\mathrm{H}(\nu)\triangleq
\sum_{\xi_{\ell}=\varnothing} \|\overgroup{\mathbf{s}_{\ell}\mathbf{s}_{\ell+1}}\|$
is the generalized length of arc~$\overgroup{\mathbf{s}_{l}\mathbf{s}_{l+1}}$
from~\eqref{eq:loss_estimation}.

\begin{algorithm}[t]
  \caption{KG-HS Algorithm.}
  \label{alg:KGHS}
  \SetAlgoLined
  \KwIn{Guiding path $\overline{\mathbf{S}}$ by~$\texttt{A}^\star$ and~\eqref{eq:astar-heuristic}.}
  \KwOut{Hybrid plan $\nu^\star$.}
  Initialize~$Q=\{\nu_{\texttt{O}}\}$; \label{alg:init}

  \While{$Q\neq \varnothing$}{
  \tcc{\textbf{Selection}}
    $\nu^\star= \textbf{argmin}_{\nu\in Q} \big\{\mathcal{C}(\nu)+\mathcal{H}(\nu)\big\}$
    by~\eqref{eq:cost-plan}; \label{alg:select}\\
    Store~$\nu^\star$ if feasible;\\
  \tcc{\textbf{Expansion}}
  $\ell=\textbf{min}\big\{\ell\in \{1,\cdots,L_{\nu^\star}\}\,|\,\xi_{\ell}=\emptyset \big\}$; \\
  \If{$\overgroup{\mathbf{s}_{\ell}\mathbf{s}_{\ell+1}}$ is collision-free}{
    $\mathbf{\Xi}=\texttt{MG-SO}(\mathbf{s}_{\ell},\, \mathbf{s}_{\ell+1})$ by~\eqref{eq:sparse-opt};\\
    $\nu^+ =\texttt{REP}\big(\nu^\star,\,\ell,\,(\mathbf{s}_\ell,\,\xi^\star)\big)$ by~\eqref{eq:replace_keyframes},\,
    $\forall \xi^\star \in \mathbf{\Xi}$;\\
    \If{$\overgroup{\mathbf{s}_\ell \mathbf{s}_{\ell+1}}$ is not allowed in~$\mathbf{\Xi}$}{
    $\bm{\kappa}_\ell =\texttt{SATA}\big(\overgroup{\mathbf{s}_\ell \mathbf{s}_{\ell+1}},\, e\big)$ by~\eqref{eq:kappa};\\
    $\nu^+=\texttt{REP}(\nu^\star,\,\ell,\,\bm{\kappa}_\ell)$ by~\eqref{eq:replace_keyframes};\\
    }
  }
  \Else{
    Select~$\mathbf{s}_{\widehat{\ell}}$ by~\eqref{eq:expand_condition}
    and compute~$\widehat{\mathbf{S}}_{\widehat{\ell}}$; \label{algline:close-by}\\
    $\nu^+=\texttt{INS}\big(\nu^\star,\,\ell,\,(\mathbf{\hat{s}},\,\emptyset)\big)$ by~\eqref{eq:add_keyframes},
    \ $\forall \mathbf{\hat{s}}\in \widehat{\mathbf{S}}_{\widehat{\ell}}$;\\
  }
  $Q\leftarrow Q\cup \{\nu^+\}$;\\
  \tcc{\textbf{Termination}}
  \If{$Q$ is empty or the time limit is reached}{
    Return the best~$\nu^\star$;\\
  }
  }
\end{algorithm}

(II) \textbf{Expansion}.
The selected node~$\nu^\star$ is expanded by finding the first keyframe $\mathbf{k}_{\ell}$ that
has not been \emph{assigned} with a mode,
i.e. $\ell=\textbf{min}\big\{\ell\in \{1,\cdots,L_{\nu^\star}\}\,|\,\xi_{\ell}=\emptyset \big\}$.
Then, the following two cases are discussed:
\textbf{(i)}
If the arc~$\overgroup{\mathbf{s}_{\ell} \mathbf{s}_{\ell+1}}$
intersects with any obstacle,
a new keyframe $(\mathbf{s}_{\widehat{\ell}},\mathbf{\varnothing})$
is generated by selecting an intermediate state $\mathbf{s}_{\widehat{\ell}}$
between $\mathbf{s}_{\ell}$ and $\mathbf{s}_{\ell+1}$ within the guiding path $\overline{\mathbf{S}}$,
which satisfies the following condition:
\begin{equation}\label{eq:expand_condition}
\big(\|\mathbf{s}_{\widehat{\ell}}-\mathbf{s}_{\ell}\|-\alpha_{\min}\big)
\big(\|\mathbf{s}_{\widehat{\ell}}-\mathbf{s}_{\ell+1}\|-\alpha_{\min}\big)>0,
\end{equation}
where~$\alpha_{\min}>0$ is a chosen parameter
such that if $\|\mathbf{s}_1-\mathbf{s}_2\|\leq\alpha_{\min}$,
the arc~$\overgroup{\mathbf{s}_1\mathbf{s}_2}$ is collision-free,
$\forall \mathbf{s}_1, \mathbf{s}_2 \in \overline{\mathbf{S}}$.
This condition ensures that
the length of both segments is either greater than $\alpha_{\min}$
or less than $\alpha_{\min}$,
thus avoiding  uneven segmentation.
Thus, a new node $\nu^+$ is generated
by inserting the keyframe~$(\mathbf{s}_{\widehat{\ell}},\,\mathbf{\varnothing})$
between~$\kappa_\ell$ and $\kappa_{\ell+1}$ within~$\nu^\star$,
i.e.,
\begin{equation} \label{eq:add_keyframes}
  \nu^+\triangleq \kappa_0 \cdots\kappa_{\ell}\, (\mathbf{s}_{\widehat{\ell}},\mathbf{\varnothing})\,
  \kappa_{\ell+1}\,\cdots\,\kappa_{L_{\nu^\star}},
\end{equation}
such that the original arc~$\overgroup{\mathbf{s}_{\ell}\mathbf{s}_{\ell+1}}$
is split into two sub-arcs~$\overgroup{\mathbf{s}_{\ell}\mathbf{s}_{\widehat{\ell}}}$
and $\overgroup{\mathbf{s}_{\widehat{\ell}}\mathbf{s}_{\ell+1}}$.
This step is encapsulated by $\nu^+\triangleq \texttt{INS}\big(\nu,\ell,(\mathbf{s}_{\widehat{\ell}},\varnothing)\big)$.
Meanwhile, to account for possible disturbances,
a set of close-by states~$\widehat{\mathbf{S}}_{\widehat{\ell}}$
are generated by adding perturbations to~$\mathbf{s}_{\widehat{\ell}}$,
e.g., small rotation and translation.
The associated set of new nodes~$\{\nu^+\}$ are all added to~$Q$;
(\textbf{ii})
If $\overgroup{\mathbf{s}_{\ell}\mathbf{s}_{\ell+1}}$
is collision-free,
a set of feasible modes is generated by solving the associated
optimization in~\eqref{eq:sparse-opt},
denoted by~$\mathbf{\Xi}(\overgroup{\mathbf{s}_{\ell}\mathbf{s}_{\ell+1}})$.
For each mode~$\xi^\star\in\mathbf{\Xi}$,
a new node is generated by replacing the keyframe~$(\mathbf{s}_{\ell}, \varnothing)\in \nu^\star$
with $(\mathbf{s}_{\ell}, \xi^\star)$, i.e.,
\begin{equation} \label{eq:replace_keyframes}
  \nu^+ \triangleq \kappa_0 \cdots \kappa_{\ell-1}\,
  (\mathbf{s}_{\ell},\,\xi^\star)\,\kappa_{\ell+1}\cdots\kappa_{L_{\nu^\star}},
\end{equation}
which is encapsulated by $\nu^+\triangleq \texttt{REP}\big(\nu,\ell,(\mathbf{s}_{\ell},\xi^\star)\big)$.
Then the set of new nodes $\{\nu^+\}$ are added to the queue~$Q$.
Meanwhile,
as discussed,
it is possible that the arc~$\overgroup{\mathbf{s}_{\ell}\mathbf{s}_{\ell+1}}$
may not be allowed in any mode~$\xi^\star\in\mathbf{\Xi}^\star$.
In this special case,
the method described in the proof of Theorem~\ref{thm:feasibility}
is applied to approximate the segment of path~$\overline{\mathbf{S}}$
by a sequence of shorter arcs with a error bound in Hausdorff distance.
For brevity, this process of approximation via sequential arc transitions
is denoted by:
\begin{equation}\label{eq:kappa}
\bm{\kappa}_\ell \triangleq \texttt{SATA}\big(\overgroup{\mathbf{s}_{\ell} \mathbf{s}_{\ell+1}},\, e\big),
\end{equation}
where~$\bm{\kappa}_\ell$ is the resulting sequence of keyframes
between~$\mathbf{s}_{\ell}$ and $\mathbf{s}_{\ell+1}$;
and~$e>0$ is the chosen bound.
Thus a new node is generated by replacing the keyframe~$(\mathbf{s}_{\ell}, \varnothing)\in \nu^\star$
with~$\bm{\kappa}_\ell$,
i.e., $\nu^+\triangleq \texttt{REP}(\nu,\,\ell,\,\bm{\kappa}_\ell)$,
which is also added to the queue~$Q$.

(III) \textbf{Termination}.
If all the keyframes in the selected node~$\nu^\star$ are assigned with modes,
it indicates that~$\nu^\star$ is a \emph{feasible} solution.
More importantly, the arc transitions in~$\nu^\star$ are collision-free,
and all allowed in the mode set~$\Xi$.
The KG-HS algorithm continues until~$Q$ is empty
or the maximum planning time is reached.
Once terminated, the current-best hybrid plan is selected and sent
to the robots for execution.

\begin{theorem}\label{thm:completeness}
    Given that the set of modes~$\Xi$ is sufficient and
    the guiding path~$\overline{\mathbf{S}}$ is collision-free,
    the proposed KG-HS algorithm finds a feasible hybrid
    plan~$\nu^\star$ in finite steps.
\end{theorem}
\begin{proof}
(Sketch)
The key insight is that any infeasible arc
is split into two arcs with shorter length,
until all the arcs are collision-free.
Therefore, there must exists a minimum arc length~$\alpha_{\min}>0$
such that any segments~$\overgroup{\mathbf{s}_{\ell}\mathbf{s}_{\ell+1}}$ in~$\nu^\star$
with length less than~$\alpha_{\min}$ is collision-free.
Consequently, the search depth is bounded by~$2|\overline{\mathbf{S}}|/\alpha_{\min}$
with~$|\overline{\mathbf{S}}|$ being the total length,
yielding that the termination condition is reached in finite steps.
Lastly, since the set of modes~$\Xi$ is sufficient,
the set of arcs~$\overgroup{\mathbf{s}_{l}\mathbf{s}_{l+1}}$ is feasible
when $\alpha_{\min}$ is sufficiently small by Theorem~\ref{thm:feasibility}.
Thus, the KG-HS algorithm is guaranteed to find a feasible solution in finite steps.
\end{proof}


\subsection{Online Execution and Adaptation}\label{subsec:control}
\change{The online execution of the hybrid plan~$\nu^\star$ consists of three parts:
(i) trajectory tracking under each arc segment with the specified mode;
(ii) mode switching between segments;
and (iii) event-triggered adaptation of the hybrid plan.}

\subsubsection{Trajectory Tracking}
\label{subsubsec:trajectory}
The nonlinear model predictive control (NMPC) framework
from~\citet{allgower2004nonlinear} is adopted for the robots
such that the object pose can track the arc
transitions~$\overgroup{\mathbf{s}_\ell\mathbf{s}_{\ell+1}}$
with mode~$\xi_\ell$.
The main motivation is to account for motion uncertainties and perturbations
both in robot motion and pushing, such as slipping and sliding.
More specifically, the tracking problem is formulated
as an online optimization problem for the pushing
forces~$\mathbf{q}_{\xi}(t)$
and the refined target trajectory~$\mathbf{s}(t)$, i.e.,
\begin{equation}\label{eq:nmpc}
\begin{split}
\underset{\{\mathbf{q}_{\xi}^{\bsss}(t),\mathbf{s}(t)\}}{\textbf{min}}&
\sum_{t\in \widetilde{\mathcal{T}}}\Big(\left\|\mathbf{q}_{\xi}^{\bsss}(t)
-\eta(\mathbf{p}^{\bsss}(t))\right\|_{_1}+
        w_{\texttt{I}}(\Delta\mathbf{p}_t)^2
        \Big)\\
        \textbf{s.t.}&\quad \mathbf{s}(0)=\widetilde{\mathbf{s}}_{0},\
        \mathbf{s}(T)=\widetilde{\mathbf{s}}_{\texttt{G}},\
        \Omega(t)\subset\widehat{\mathcal{W}};\\
        &\quad \mathbf{s}(t+1)=\mathbf{s}(t)+\mathbf{R}_{\bsss}\mathbf{p}^{\bsss}(t),
        \ \forall t\in \widetilde{\mathcal{T}};\\
        &\quad \mathbf{q}_{\widetilde{\xi}}^{\bsss}(t)\in
        \mathrm{Q}_{\widetilde{\xi}}^{\bsss},\
        \forall t\in \widetilde{\mathcal{T}};
\end{split}
\end{equation}
where~$\widetilde{T}$ is the planning horizon of NMPC;
$\widetilde{\mathcal{T}}\triangleq \{1,\cdots,\widetilde{T}\}$;
$\Delta\mathbf{p}_t\triangleq \mathbf{p}^{\bsss}(t)-\mathbf{p}^{\bsss}(t-1)$
is the acceleration of the object;
$w_{\texttt{I}}>0$ is a design parameter;
$\widetilde{\mathbf{s}}_{0}$ is current state of the target;
and $(\widetilde{\mathbf{s}}_{\texttt{G}},\, \widetilde{\xi})$
are the local reference state and the associated mode,
selected from the hybrid plan~$\nu^\star$.
The second objective measures the smoothness of the resulting trajectory,
which is ignored in quasi-static analysis but considered here.


The optimization above can be solved by a nonlinear optimization solver,
e.g.,~\texttt{IPOPT} within CasADi from~\citet{andersson2019casadi}.
Denote by~$\mathbf{S}^\star(t)\triangleq \mathbf{s}_1\cdots \mathbf{s}_{\widetilde{T}}$
the resulting trajectory, of which the first state
is chosen as the subsequent desired state~$\widehat{\mathbf{s}}$.
The associated desired contact point and orientation
for robot~$R_n$ are given by~$\widehat{\mathbf{c}}_n$
and~$\widehat{\psi}_n$, $\forall n \in \mathcal{N}$.
Thus, their control inputs are given by
$\mathbf{v}_n=K_{v}(\widehat{\mathbf{\mathbf{c}}}_n-\mathbf{c}_n)$
and~$\omega_n=K_{\psi}(\widehat{\psi}_n-\psi_n)$,
where~$K_{v}$ and $K_{\psi}$ are positive gains;
and $\mathbf{c}_n, \psi_n$ are the current contact
point and orientation of~$R_n$.
With simple P-controllers and a small step size,
the robots can push the target along~$\mathbf{S}^\star(t)$,
which is updated by solving~\eqref{eq:nmpc} recursively
online at a lower rate.

\subsubsection{Mode Switching}
\label{subsubsec:switch}
Initially, since the robots are homogeneous,
they are assigned to their closest contact
points given the first mode~$\xi_0$.
Afterwards, as the segment~$\overgroup{\mathbf{s}_\ell\mathbf{s}_{\ell+1}}$ is traversed,
the fleets switch from mode~$\xi_\ell$ to~$\xi_{\ell+1}$.
To reduce switching time and avoid possible collisions,
the switching strategy is designed as follows.
Let the set of old contact points be~$\{\mathbf{c}_{n}\}$ for mode~$\xi_{\ell}$ and
the new set~$\{\mathbf{c}^+_n\}$ for mode~$\xi_{\ell+1}$.
First, the periphery of the object is transformed into a circle
such that the relative ordering of the contact points is preserved.
Second, an diameter associated with
angle $\theta^\star\in [0,\, \pi]$ is chosen
such that the number of old contact points
and the number of new contact points
on both side of the diameter are equal.
Consequently, the points in~$\{\mathbf{c}^+_n\}$
are numbered clockwise starting from~$\theta=\theta^\star$,
while the same applies for the robots at~$\{\mathbf{c}_{n}\}$.
Lastly, each robot is assigned to the new contact point
with the same number and a local path is planned for navigation.
It can be shown that no collisions between the robots can happen
and the length of each path is
not longer than half of the perimeter,
if 
the minimum distance between the target and obstacles
is greater than the robot radius;
and all robots move at similar speed,
as illustrated in Fig.~\ref{fig:overall}.
\change{\subsubsection{Online Adaptation} \label{subsubsec:adaptation}
  Due to model uncertainties such as communication delays
  and actuation noises or slipping,
  the NMPC scheme in~\eqref{eq:nmpc} might not be able to rectify
  the system towards the reference state,
  i.e., the object can not follow the refined trajectory~$\mathbf{S}^\star$.
  In this case, a re-planning mechanism for the hybrid plan~$\nu^\star$ is necessary
  to find a new guiding path and a new associated hybrid plan.
  The triggering condition is rule-based:
  (i) the target object deviates from the planned trajectory beyond a given threshold;
  (ii) the distance between the object and obstacles is less than the size of the robot;
  (iii) the object is stuck within an area for a certain period;
  (iv) new static obstacles are detected;
  or (iv) one or several robots fail and can not participate in the task anymore.
  Then, as shown in Fig.~\ref{fig:overall}, the KG-HS algorithm is executed again
  with the current system state as the initial state and the functional robots,
  yielding a new hybrid plan~$\nu^\star$.
  The NMPC is recomputed to generate new reference position for the robots.
  Similar approaches can be applied to the case of dynamic obstacles,
  where the safety constraint in the NMPC can modify the~$\mathbf{S}^\star$
  given the perceived or predicted motion of the obstacles.
  As validated by numerical results in Sec.~\ref{sec:experiments},
  this mechanism is effective and essential for improving the overall robustness,
  especially in uncertain environments or during robot failures.
  Last but not least, if measurements of the object state are not always available
  or temporary lost. An extended Kalman filter (EKF) can be adopted to estimate the object
  state at all time, i.e., given the desired target velocity,
  the contact points,
  and the state of all robots.
  Detailed formulations are omitted here for brevity.}

  \begin{figure}[t!]
    \centering
    \includegraphics[width=0.45\textwidth]{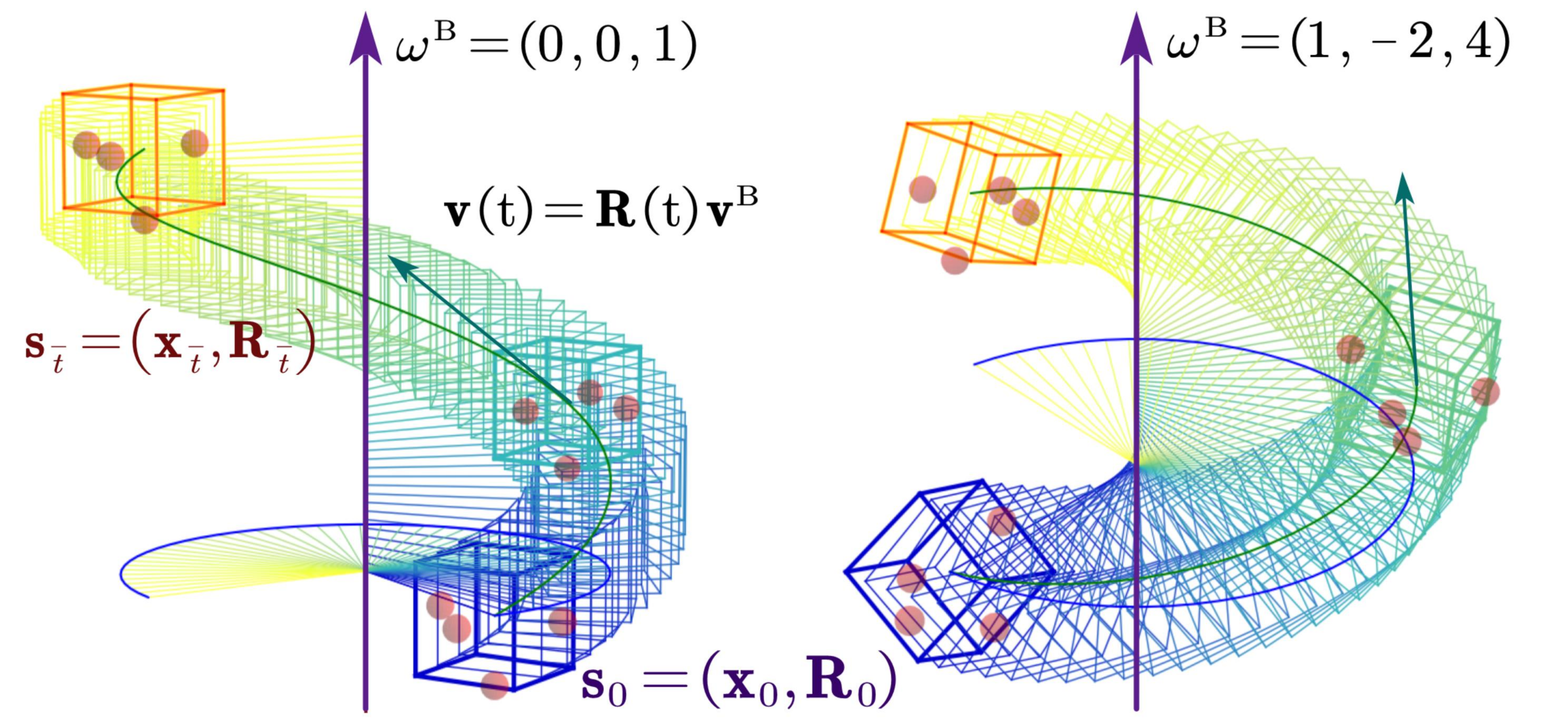}
    \caption{\change{Illustration of the spiral transition in 3D workspace}.}
    \label{fig:spiral_trajectory}
    \vspace{-5mm}
  \end{figure}

\begin{figure*}[t!]
        \centering
        \includegraphics[width=1\linewidth]{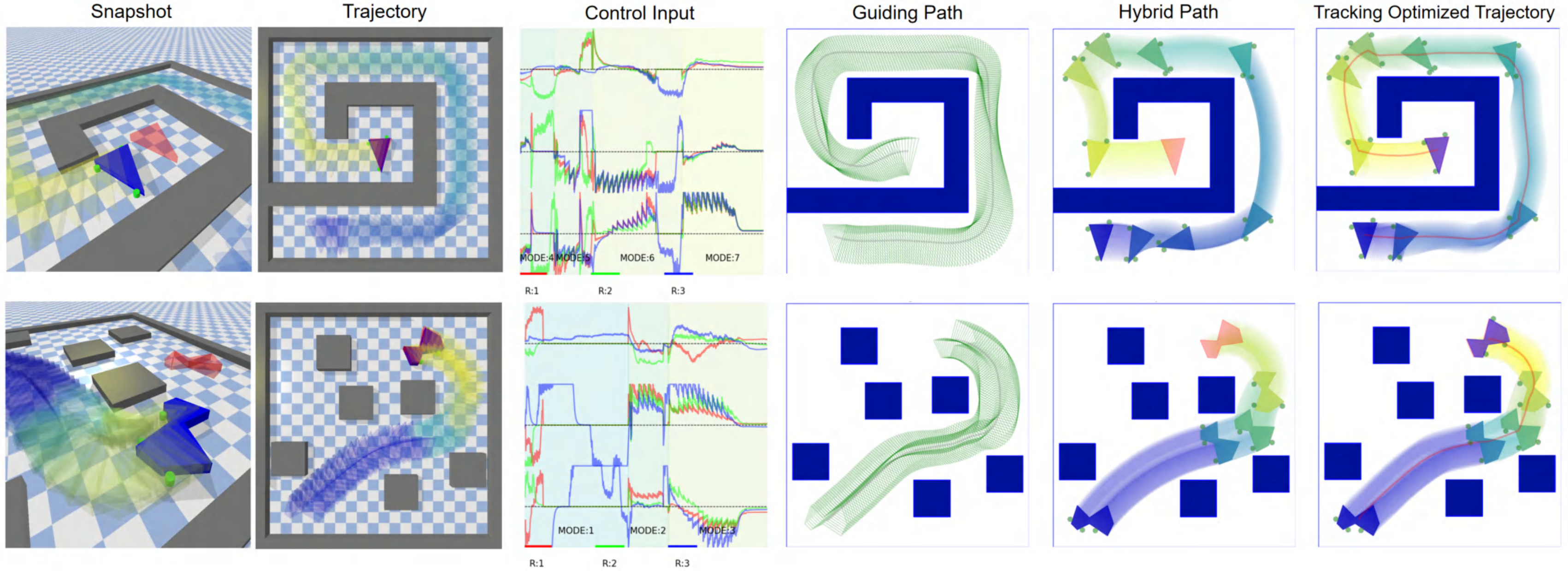}
        \caption{\change{Simulation results of three robots pushing
        different target objects (goal state in red),
        including the guiding path~$\overline{\mathbf{S}}$,
        the hybrid plan~$\nu^\star$ with keyframes and modes,
        the final trajectory~$\mathbf{S}^\star$ after execution,
        and the control inputs of the robots.}}
        \label{fig:sim_results}
        \vspace{-6mm}
\end{figure*}


\subsection{Discussion}\label{subsec:overall}

\subsubsection{Computation Complexity}
\label{subsubsec:complexity}
The time complexity of the proposed KG-HS algorithm
is mainly determined by the mode generation, the collision checking
and the size of the search space.
\change{The MG-SO method in~\eqref{eq:sparse-opt} for mode generation is a LP,
with the variable size proportional to the number of all contact points $N_{\mathcal{V}}$.}
Thus, it has a complexity of~$\mathcal{O}(N_ {\mathcal{V}}^{3.5})$ from~\citet{terlaky2013interior}.
The time complexity of the GJK algorithm for 2D collision checking
from~\citet{bergen1999fast}
is $\mathcal{O}(V+V_O)$,
where~$V_O$ is the maximum number of vertices for the obstacles.
Meanwhile, for each iteration of node expansion in Alg.~\ref{alg:KGHS},
at most~$W_{k}$ keyframe variations in the first case or $W_{\xi}$ modes
in the second case are generated.
\change{
Meanwhile,
the computation of the loss $J_{\text{MF}}$ in Eq.~\eqref{eq:cost-plan}
is a LP with the variable size proportional to the number of robots,
so the complexity is $\mathcal{O}(N^{3.5})$.
Recall that the search depth is bounded by~$2\alpha/\alpha_{\min}$
from~Theorem~\ref{thm:completeness}.
Therefore, the total time complexity of the hybrid search algorithm
is approximated by~$\mathcal{O}\Big((N_ {\mathcal{V}}^{3.5}+N^{3.5}+V+V_{O})\cdot \big(\max\{W_\xi,W_k\}\big)^{2\alpha/\alpha_{\min}}\Big)$.
More numerical analyses on how the proposed algorithm scales
across different fleet sizes can be found in Sec.~\ref{subsubsec:scale}.}

\subsubsection{Additional Heuristics}
\label{subsubsec:discussion}
Some additional heuristics are applied to further
improve the search performance:
{(I) Selection of new keyframes}.
Instead of simply choosing the middle point between~$\mathbf{s}_{\ell}$
and~$\mathbf{s}_{\ell+1}$ in~\eqref{eq:add_keyframes},
the new keyframe~$\mathbf{s}_{\hat{\ell}^\star}$ can be selected
by minimizing the distance between the arc transitions
($\overgroup{\mathbf{s}_{\ell}\mathbf{s}_{\hat{\ell}^\star}}$,
$\overgroup{\mathbf{s}_{\hat{\ell}^\star}\mathbf{s}_{\ell+1}}$)
and the obstacles,
typically yielding the tuning points on the guiding path;
{(II) Local reachability}.
The close-by states~$\widehat{\mathbf{S}}_{\widehat{\ell}}$
in Line~\ref{algline:close-by} of Alg.~\ref{alg:KGHS}
are first filtered to ensure all states inside are reachable
via a collision-free arc from $\mathbf{s}_{\ell}$,
to avoid adding infeasible nodes in the search tree;
{(III) Inflation of obstacles}.
The static obstacles are inflated by at least the robot radius
when generating the guiding path.
Moreover,  a differentiable penalty is added
in the objective of~\eqref{eq:nmpc}
for a sufficient safety margin of the robots.
\subsubsection{Generalization}\label{subsubsec:general}
The proposed method can be generalized to different scenarios:
\change{{(I) Heterogeneous robots}.
When the robots are heterogeneous with different capabilities,
the interaction mode should be defined as
$\xi\triangleq(\mathbf{c}_1,R_{c_1})\cdots(\mathbf{c}_N,R_{c_N})$,
which assigns specific robot to each contact point~$\mathbf{c}_n$
based on the associated force constraints.
For instance, robots with larger maximum forces are assigned to
contact points that require larger forces,
while robots with smaller forces are assigned to assistive contact points
that provide additional geometric constraints;
{(II) 3D pushing tasks}.
If the task is to push a target object with 6D pose in
a complex~3D scene,
the proposed framework can be adopted with minor modifications,
under the condition that
the workspace is zero-gravity and resistant,
due to the quasi-static assumption.
More specifically,
consider a target object of which the state is denoted
by~$\mathbf{s}(t)\triangleq (\mathbf{x}(t),\mathbf{R}(t))$,
where~$\mathbf{x}(t)\in \mathbb{R}^3$ is the position
and $\mathbf{R}(t)\in \mathbb{R}^{3\times 3}$ is the rotation matrix
for orientation.
Additionally, the robots are modeled as spheres with radius~$r$,
with the translational velocity~$\mathbf{v}_n\in \mathbb{R}^3$
as control inputs.
When moving dynamically,
the target experiences a force of resistance that
is proportional to its generalized velocity, i.e.,
$\mathbf{q}_{\texttt{res}} \triangleq
- \mathbf{K}_{\texttt{res}}\mathbf{p} \in \mathbb{R}^6$
with~$\mathbf{K}_{\texttt{res}}$ being a diagonal matrix,
  and $\mathbf{p}=(\mathbf{v},\bm{\omega})=(v_x,v_y,v_z,\omega_x,\omega_y,\omega_z)\in \mathbb{R}^6$ as the generalized velocity.
  Similar to the 2D case,
  consider that the robots push the target with
  velocity~$\mathbf{p}(t)=
  \text{diag}(\mathbf{R}(t),\mathbf{R}(t))\mathbf{p}^{\bsss}$
  for a time duration~$\bar{t}>0$.
  Then, the arc transition in~\eqref{eq:arc_trajectory}
  can be extended to~3D workspace by:
  \begin{equation}\label{eq:spiral_trajectory}
    \begin{split}
    \text{d}\mathbf{R}(t)& \triangleq \bm{\omega}(t) \times \mathbf{R} (t)\,\text{d}t=(\mathbf{R}_0\,\bm{\omega}^{\bsss})\times \mathbf{R}(t)\,\text{d}t,\\
    \mathbf{R}(t) &= \mathbf{P}_{\bm{\omega}}
          \left[ \begin{matrix}
          \mathbf{Rot}\left( \|\bm{\omega}^{\bsss}t\|_2 \right) &		0\\
          0&		                                                    1\\
          \end{matrix} \right]
    \mathbf{P}_{\bm{\omega}}^{-1} \mathbf{R}_0,\\
    \Delta \mathbf{x} &\triangleq \big(\int_{t=0}^{\bar{t}} \mathbf{R}(t) dt\big)\, \mathbf{v}^{\bsss}
    =\mathbf{\Phi}_{\texttt{3D}}(\mathbf{R}_0,\, \bm{\omega}^{\bsss}\bar{t})\ \bar{t}\ \mathbf{v}^{\bsss},
    \end{split}
  \end{equation}
  where $\mathbf{P}_{\bm{\omega}}\in \mathbb{R}^{3\times 3}$ is an unitary matrix
  of which the $3$-th column is aligned with the rotation axis $\bm{\omega}$,
   i.e. $\mathbf{P}_{\bm{\omega}}(3)=\mathbf{R}_{0}\bm{\omega}^{\bsss}$.
   As shown in Fig.~\ref{fig:spiral_trajectory},
   this motion results in a spiral trajectory in 3D space,
   i.e.,~a combination of a constant rotation about the axis $\bm{\omega}$
   and a translation along $\bm{\omega}$.
\begin{lemma}\label{lm:spiral_existence}
  For any given pair of initial state~$\mathbf{s}_0$ and goal state~$\mathbf{s}_{\texttt{G}}$ of the object in the 3D workspace,
  there exists a spiral trajectory $\mathbf{Trj}(\mathbf{s}_0,\mathbf{p}^{\bsss},\bar{t})$
  such that $\|\bm{\omega} \bar{t}\|_2\in \left[0,\pi \right]$.
\end{lemma}
\begin{proof}
  (Sketch) This is a straightforward corollary of the Chasles's theorem~\cite{whittaker1964treatise}.
  Analogous to the 2D case,
  the desired velocity $\mathbf{p}^{\bsss}=(\mathbf{v}^{\bsss},\bm{\omega}^{\bsss})$ can be determined
  by~\eqref{eq:spiral_trajectory},
  yielding the spiral trajectory $\mathbf{Trj}(\mathbf{s}_0,\mathbf{p}^{\bsss},\bar{t})$.
\end{proof}
Moreover,
in total $D_{\texttt{3D}}=12$ directions are chosen to evaluate the
multi-directional feasibility~$J_{\texttt{MF}}$.
Similarly, the mode generation in~\eqref{eq:sparse-opt}
should adapt to the constraints resulting from the~3D friction cone,
by allowing~$3$ degrees of freedom for the contact force at each point.
Consequently, the number of optimization variables
becomes~$N_{\mathcal{V}}\times 3D_{\texttt{3D}}$.
Note that the hybrid search algorithm KG-HS still applies,
with the 6D statespace~$\mathcal{S}$.
Lastly, the control input of each robot~$R_n$ is adjusted to be
$\mathbf{v}_n=K_{v}(\widehat{\mathbf{\mathbf{c}}}_n-\mathbf{c}_n)$,
$\forall n\in \mathcal{N}$.
}
 \begin{figure*}[t!]
  \centering
  \includegraphics[width=0.96\linewidth]{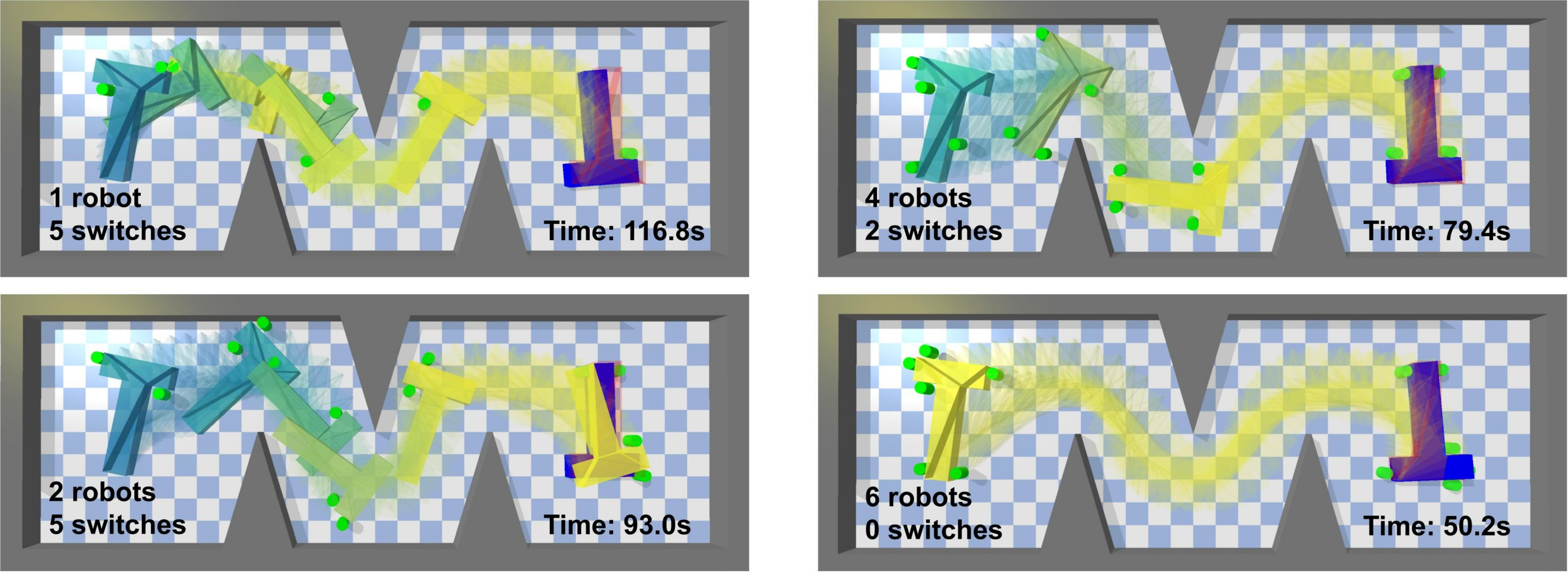}
  \vspace{-0mm}
  \caption{
    \change{Final trajectories when different number of robots
      are deployed,
      yielding different modes and task durations.}}
  \label{fig:robotnum}
  \vspace{-2mm}
\end{figure*}
\begin{figure*}[t!]
  \centering
  \includegraphics[width=0.96\linewidth]{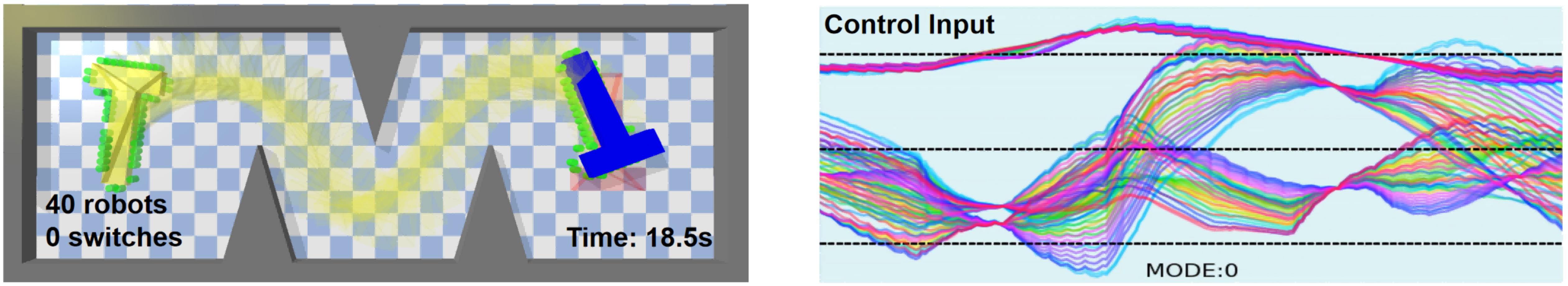}
  \vspace{-0mm}
  \caption{\change{Scalability analysis for~$40$ robots: final trajectories (\textbf{Left}) and control inputs of all robots (\textbf{Right}).}}
  \label{fig:scalability_40robots}
  \vspace{-4mm}
\end{figure*}


\section{Numerical Experiments} \label{sec:experiments}
To further validate the proposed method,
extensive numerical simulations and hardware experiments
are presented in this section.
The proposed method is implemented in Python3 and tested on a
laptop with an Intel Core i7-1280P CPU.
Numerical simulations are conducted in the \texttt{PyBullet} environment
as a high-fidelity physics engine from~\citet{coumans2019},
while~\texttt{ROS1} is adopted for the hardware experiments.
The \texttt{GJK} package from~\citet{bergen1999fast} is used for collision checking,
and the \texttt{CVXOPT} package from~\citet{diamond2016cvxpy} for solving liner programs.
Simulation and experiment videos are available
in the supplementary material.

\subsection{Numerical Simulations}

\subsubsection{System Setup}
\change{
Three distinct scenarios of size~$20m \times 20m$
with concave and convex obstacles
are considered,
including the first scenario as shown in Fig.~\ref{fig:overall},
and the second and third scenarios in Fig.~\ref{fig:sim_results}.
The target object has different shapes and sizes in each scenario,
with a mass of $10kg$,
a ground friction coefficient of $0.5$,
a side friction coefficient of $0.2$,
and a maximum thrust of $30\text{N}$.
Three homogeneous robots with a diameter of $0.25m$ are adopted for the pushing task.
}
More specifically,
a long rectangular object should be pushed through
a narrow passage in the first scenario;
a triangular object should be pushed following
a spiral corridor in the second scenario;
and a concave object should pass through
cluttered rectangular pillars in the last scenario.
The task is considered completed if the distance between the target object
and the goal is less than $0.2m$,
after which the robots can continue improving the accuracy.
The Pybullet simulator runs at a frequency of $240$Hz.
The MPC-based online optimization for object trajectory
updates at a frequency of $10Hz$, with a time horizon of $2s$
and an average speed of $0.5m/s$.
The frequency to update the intermediate reference for MPC is $1Hz$.
Other parameters are set as follows:
the weight for multi-directional feasibility
  $[w_{d}]=[5,1,1,1,1,1]$ in~\eqref{eq:feasibility},
 the weight $w_{\texttt{t}}=10$ in~\eqref{eq:cost-plan},
 the weight $w_{\texttt{I}}=10^4$ in~\eqref{eq:nmpc},
 the minimum split length~$\alpha_{\min}=0.1m$ in~\eqref{eq:feasibility}.

\subsubsection{Results}
As is shown in Fig.~\ref{fig:overall} and Fig.~\ref{fig:sim_results},
the proposed method successfully completed all tasks
in three distinctive scenarios.
The intermediate results are also shown including the guiding path,
the hybrid plan (with keyframes and modes highlighted in different colors),
the resulting trajectory of tracking the hybrid plan
via the proposed online NMPC,
and the control inputs of all robots during execution.
The planning time of three scenarios is $12.3s$, $13.2s$ and $21.1s$,
while the execution time is $79.2s$, $117s$ and $64.1s$.
The three tasks take $3$, $7$ and $3$ mode switches,
which is due to the fact that more rotations are required
for the second scenario.
The proposed NMPC can effectively track the optimized trajectory,
with the average tracking errors being $0.025m$,$0.026m$,and $0.048m$.
More metrics are discussed in the comparisons below.
Moreover, it is interesting to note that despite the NMPC problem in~\eqref{eq:nmpc}
can be solved in approximately~$40ms$,
yielding a maximum rate of $25Hz$,
higher updating rates for NMPC does not significantly improve control performance.
Instead,
the tracking control is more sensitive to the update rate
of the intermediate goal selected from the hybrid plan.
If the goal is chosen too close,
the resulting trajectory might have
drastic changes of the velocity,
while a too far goal might lead to a delayed response to errors.
Therefore, a lower rate to update the intermediate goal
would allow a more gradual convergence.
Lastly, it is worth mentioning that in all three scenarios
the online execution of the hybrid plan does not require any replanning,
different from the hardware experiments.

\subsubsection{\change{Complexity and Scalability}}\label{subsubsec:scale}
\change{To evaluate the computational complexity and scalability
of the proposed method,}
a larger environment of $30m\times10m$ is considered
where the target is a T-shape object with a mass of~$5kg$
and a ground friction coefficient of~$0.5$.
Since the maximum pushing force of one robot is~$30N$,
a single robot can move this object.
As shown in Fig.~\ref{fig:robotnum},
the proposed method is applied to robotic fleets with $1$, $2$, $4$ and $6$ robots
for the same task,
where the robots are randomly initialized.
The planning time is $10.3s$, $13.5s$, $15.2s$ and $18.1s$,
while the execution time is $116.8s$, $93.0s$, $79.4s$ and $50.2s$.
The resulting hybrid plans and trajectories are shown,
where the required number of modes are $6$, $6$, $3$ and $1$, respectively;
and the average tracking errors are $0.08m$, $0.05m$, $0.01m$ and~$0.01m$.
It indicates that fewer robots often require more mode switches
to complete the task, while more than~$6$ robots only need one mode.


\usetikzlibrary{patterns}
\pgfplotstableread[col sep=comma]{
	type,   Guiding Path,  KG-HS,    MPC,   Execution
	2,      3.78,          2.45,    9.32,   63.5
  5,      3.80,          3.67,    0.68,   42.4
	10,     3.75,          4.16,    0.06,   28.4
	20,     4.09,          3.82,    0.07,   24.5
  40,     3.79,          6.35,    0.08,   18.5
	}\mytable
\pgfplotsset{/pgfplots/error bars/error bar style={thick, black}}
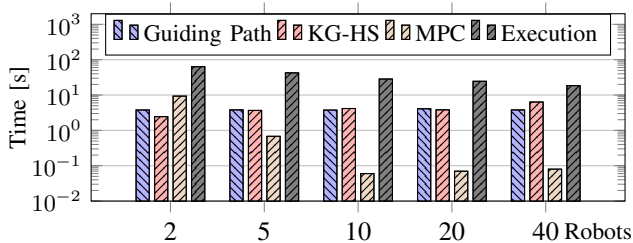
\begin{figure}[t!]

  \begin{subfigure}{}
    \centering
    \begin{tikzpicture}
      \begin{semilogyaxis}[
        width=0.98\linewidth,
        ybar,
        bar width=5pt,
        log origin = infty,
        ymin=0.01,
        ymax=2000,
        ytick={0.01,0.1, 1, 10, 100, 1000},
        yticklabel style = {font=\small},
        minor y tick num=9,
        height=0.46\linewidth,
        ymajorgrids=true,
        enlarge x limits={abs=30pt},
        legend style={at={(0.5,1.00),font=\small},
        anchor=north,legend columns=4, font=\small},
        ylabel={Time [s]},
        y label style={at={(axis description cs:0.075,0.5)},anchor=south, font=\small},
        symbolic x coords={2, 5,10, 20, 40},
        xlabel={Robots},
        x label style={at={(axis description cs:0.95,0.145)},font=\small},
        xtick=data,
      ]

          \addplot+[draw=black,fill,postaction={
            pattern=north west lines
        }] plot [error bars/.cd, y dir=both, y explicit,
                 error mark options={line width=0.25pt,rotate=90}]
           table[x=type,y=Guiding Path]{\mytable};
          \addlegendentry{Guiding Path}

          \addplot+[draw=black,fill,postaction={
            pattern=north east lines}]
            plot [error bars/.cd, y dir=both, y explicit,
                 error mark options={line width=0.25pt,rotate=90}]
           table[x=type,y=KG-HS]{\mytable};
          \addlegendentry{KG-HS}

          \addplot+[draw=black,fill,postaction={
            pattern=north east lines}]
            plot [error bars/.cd, y dir=both, y explicit,
                 error mark options={line width=0.25pt,rotate=90}]
            table[x=type,y=MPC]{\mytable};
          \addlegendentry{MPC}

          \addplot+[draw=black,fill,postaction={
            pattern=north east lines}]
            plot [error bars/.cd, y dir=both, y explicit,
                 error mark options={line width=0.25pt,rotate=90}]
            table[x=type,y=Execution]{\mytable};
          \addlegendentry{Execution}

        \end{semilogyaxis}
      \end{tikzpicture}
    \end{subfigure}
    \vspace{-4mm}
    \caption{
      \change{Computation time of each component across different fleet sizes.}
    }
    \label{tab:computation_time}
    \vspace{-2mm}
\end{figure}

\change{Furthermore, to demonstrate the
applicability to even larger robot teams,
the radius of the robots is reduced to~$0.15m$
and the maximum pushing force to~$15N$.
Under the same setting,
the proposed method is applied again to $5$, $10$, $20$ and $40$ robots,
of which the computation time of the main components are summarized in Fig.~\ref{tab:computation_time}.
As expected, to search for the guiding path is independent of the fleet size (around $3.7s$).
The computation time of KG-HS aligns with our complexity analysis in Sec.~\ref{subsubsec:complexity},
which takes $2.5s$ for $2$ robots and $6.4s$ for $40$ robots.
The cost of MG-SO and GJK is independent of the number of robots,
so the increase in time primarily stems from the $O(N^{3.5})$
complexity of computing~$J_{\text{MF}}$.
Lastly, it is interesting to note that the computation time for NMPC
decreases as the number of robots increases.
This is because the arc transition~$\mathbf{Trj}(\mathbf{s}_0,\mathbf{p}^{\bsss},\bar{t})$
for any planned mode can already be tracked accurately with enough robots,
which complies with the intuition that the pushing task is easier with more robots.
This leads to an offline planning time of only around $20s$ for~$40$ robots
and near-zero online planning time, with a high tracking accuracy
as shown in Fig.~\ref{fig:scalability_40robots}.
}

\begin{figure}[t!]
  \centering
  \includegraphics[width=0.95\linewidth]{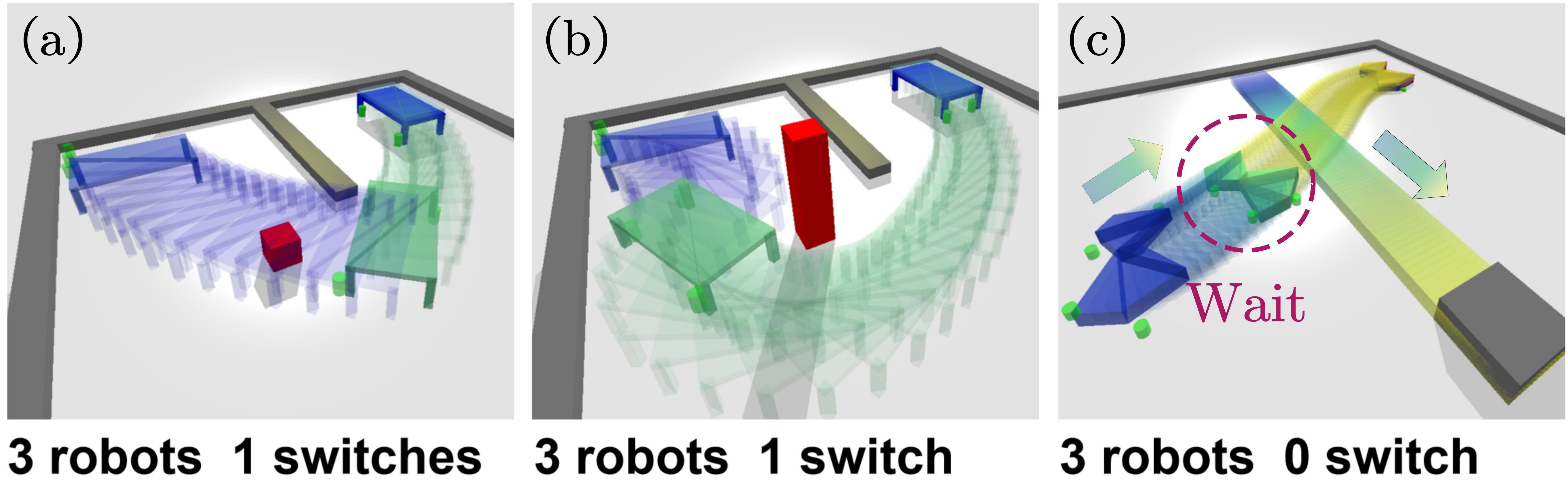}
  \vspace{-0mm}
  \caption{\change{Online adaptation under unknown obstacles (in red)
      and dynamic obstacles (in shaded grey).}}
  \label{fig:unknown_and_moving}
  \vspace{-5mm}
\end{figure}

\subsubsection{\change{Adaptability}}
\change{To further demonstrate the adaptability of the proposed method,
 the following scenarios are evaluated.}

\change{{(I) \textbf{Unknown and dynamic obstacles}}.
Since some obstacles might be unknown during offline planning,
it is important to adapt to them during online execution.
Fig.\ref{fig:unknown_and_moving} shows two cases that might appear:
(a) the guided path and refined trajectory of the object does not collide
with the small obstacle;
and (b) a potential collision is detected with a larger obstacle,
and a new hybrid plan is obtained
via the adaptation scheme described in Sec.~\ref{subsubsec:adaptation}.
Moreover, in case of a dynamic obstacle,
the current pose and predicted trajectory of the obstacle is incorporated
in the NMPC constraints in~\eqref{eq:nmpc}.
Consequently, as shown in Fig.\ref{fig:unknown_and_moving},
the robots slow down and wait for the obstacle to pass,
while following the \emph{same} hybrid plan.}

\begin{figure}[t!]
  \centering
  \includegraphics[width=0.95\linewidth]{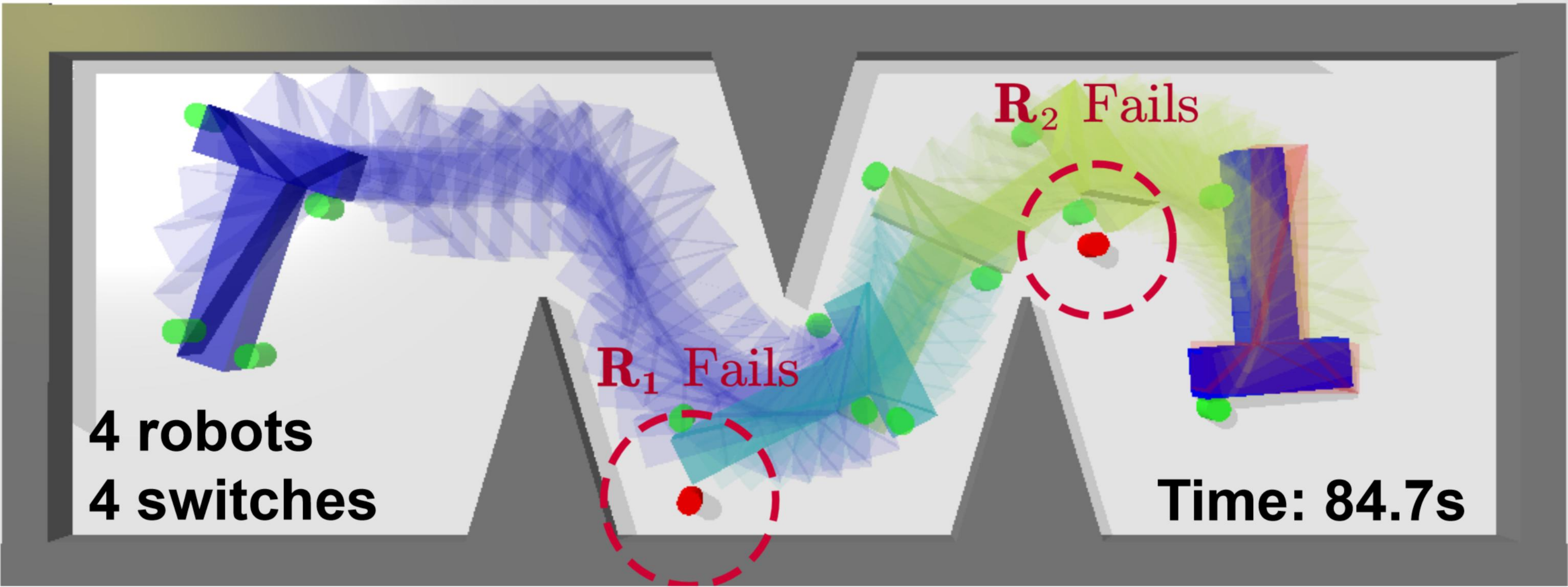}
  \vspace{-0mm}
  \caption{\change{Plan adaptation when two robots (in red) fail consecutively
  during execution}.}
  \label{fig:breakdown}
  \vspace{-4mm}
\end{figure}

\change{{(II) \textbf{Robot failures}}.
As the robots fail, the proposed online adaptation is essential for the system to recover.
Fig.\ref{fig:breakdown} shows the same pushing task as in Fig.~\ref{fig:sim_results}
with~$4$ robots, but $2$ of them fail consecutively at $25.2s$ and $58.3s$.
Via the replanning scheme, a new hybrid plan is generated each time a robot fails.
The task is successfully completed with the remaining $2$ robots at time~$84.7s$,
which however requires $3$ mode switches (compared with~$2$ switches in the nominal case).
It is worth noting that since the failed robots are treated as obstacles during re-planning,
the resulting hybrid plan avoids collision with these robots.}
\begin{figure}[t!]
  \centering
  \includegraphics[width=1.0\linewidth]{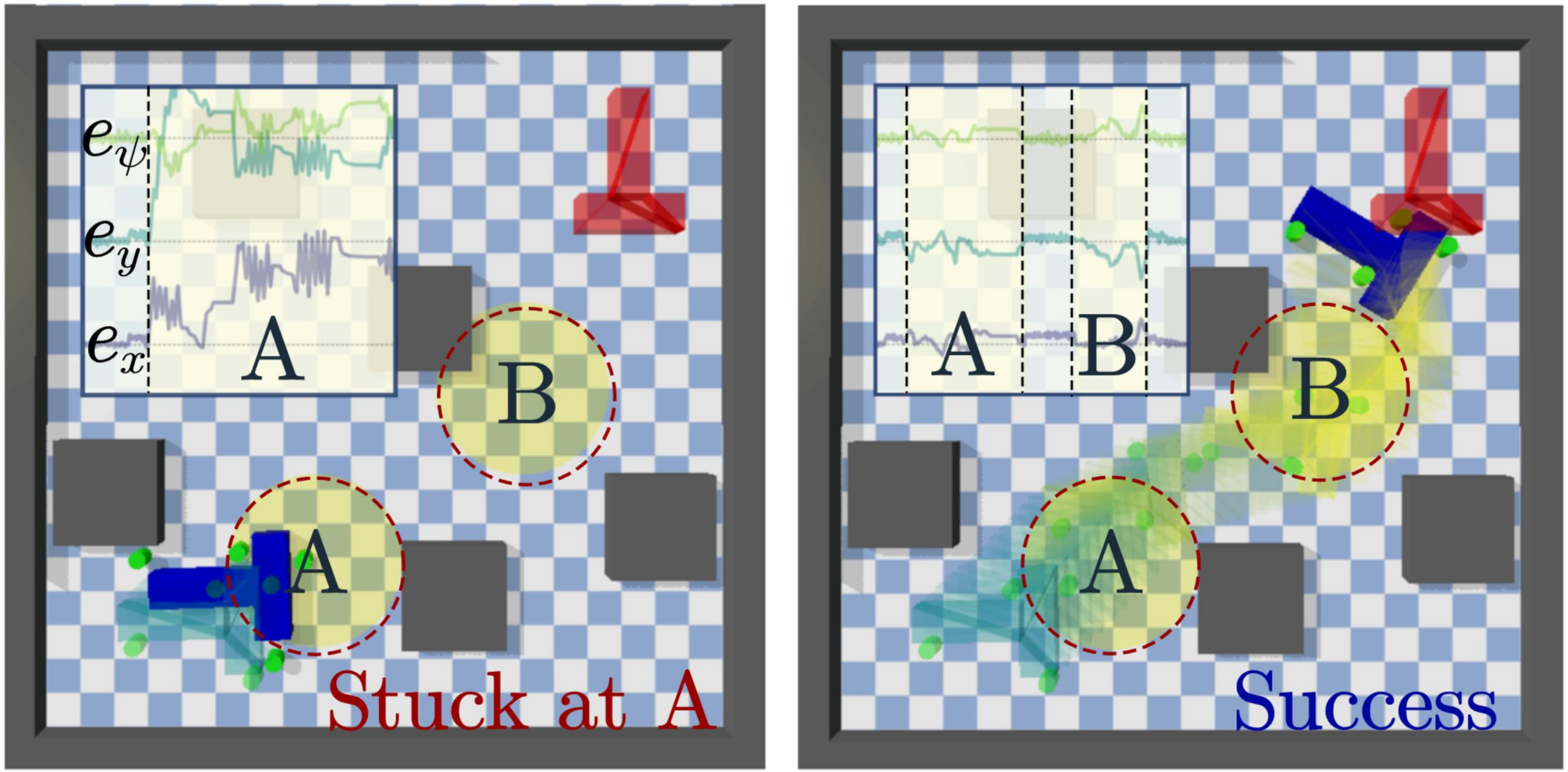}
  \vspace{-2mm}
  \caption{\change{Robustness to temporary ``blackout'' (in Regions~$A$ and $B$),
      where no measurements of the object state are obtained.
Final trajectories and state estimation error are shown
  by linear extrapolation (\textbf{Left}) and constrained EKF (\textbf{Right}).}}
  \label{fig:observation}
  \vspace{-4mm}
\end{figure}
\begin{figure*}[t!]
  \centering
  \includegraphics[width=0.975\linewidth]{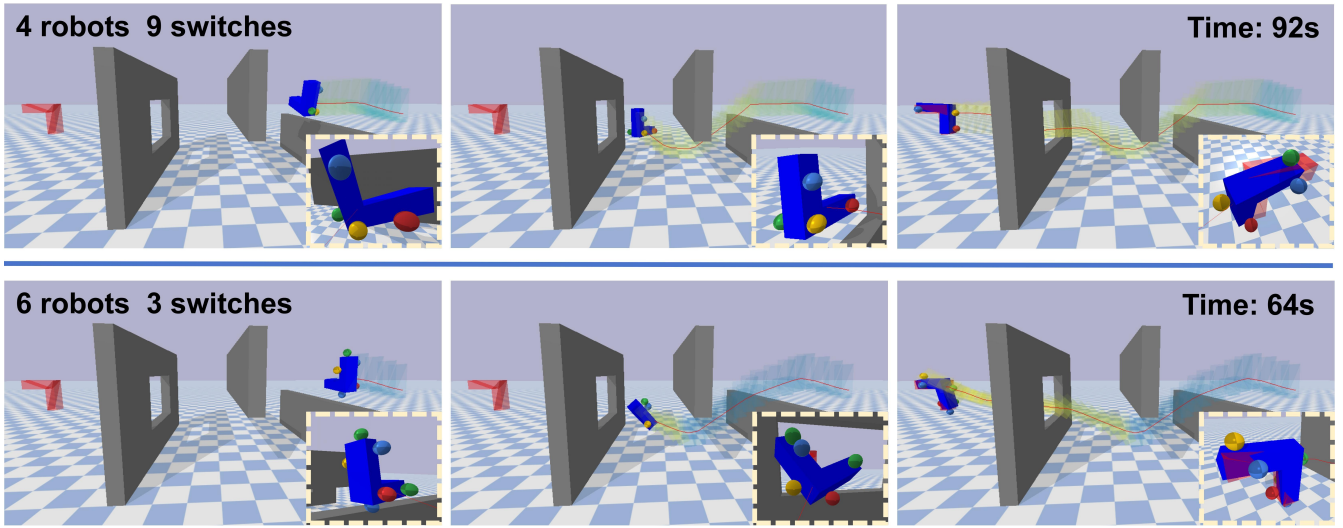}
  \vspace{-0mm}
  \caption{\change{Snapshots of 3D pushing tasks, where~$4$ robots (\textbf{Top})
      and~$6$ robots (\textbf{Bottom}) are deployed.}}
  \label{fig:3d_scenarios}
  \vspace{-4mm}
\end{figure*}

\change{{(III) \textbf{Temporary ``blackout''}}.
  Accurate observation of the state of the target object might not be always possible,
  e.g., the localization system for the object might be occluded temporarily thus no measurement
  of the object state is available.
  As described in Sec.~\ref{subsubsec:adaptation},
  an EKF is adopted to estimate the object state, given the states of the robots
  and the current pushing mode.
  Fig.\ref{fig:observation} shows the scenario where the objects can not be measured
  within two circular areas on the guided path.
  The proposed estimator has an accuracy of~$0.05m$ in position and $0.05rad$ in orientation,
  which leads to a successful task under $26\%$ blackout in time.
  In contrast, significant mismatch is generated via simple linear extrapolation,
  yielding a trajectory that is completely off in~\eqref{eq:nmpc}.
  Consequently, the robots deviate quickly from the correct path.}

\begin{figure}[t!]
  \centering
  \includegraphics[width=0.95\linewidth]{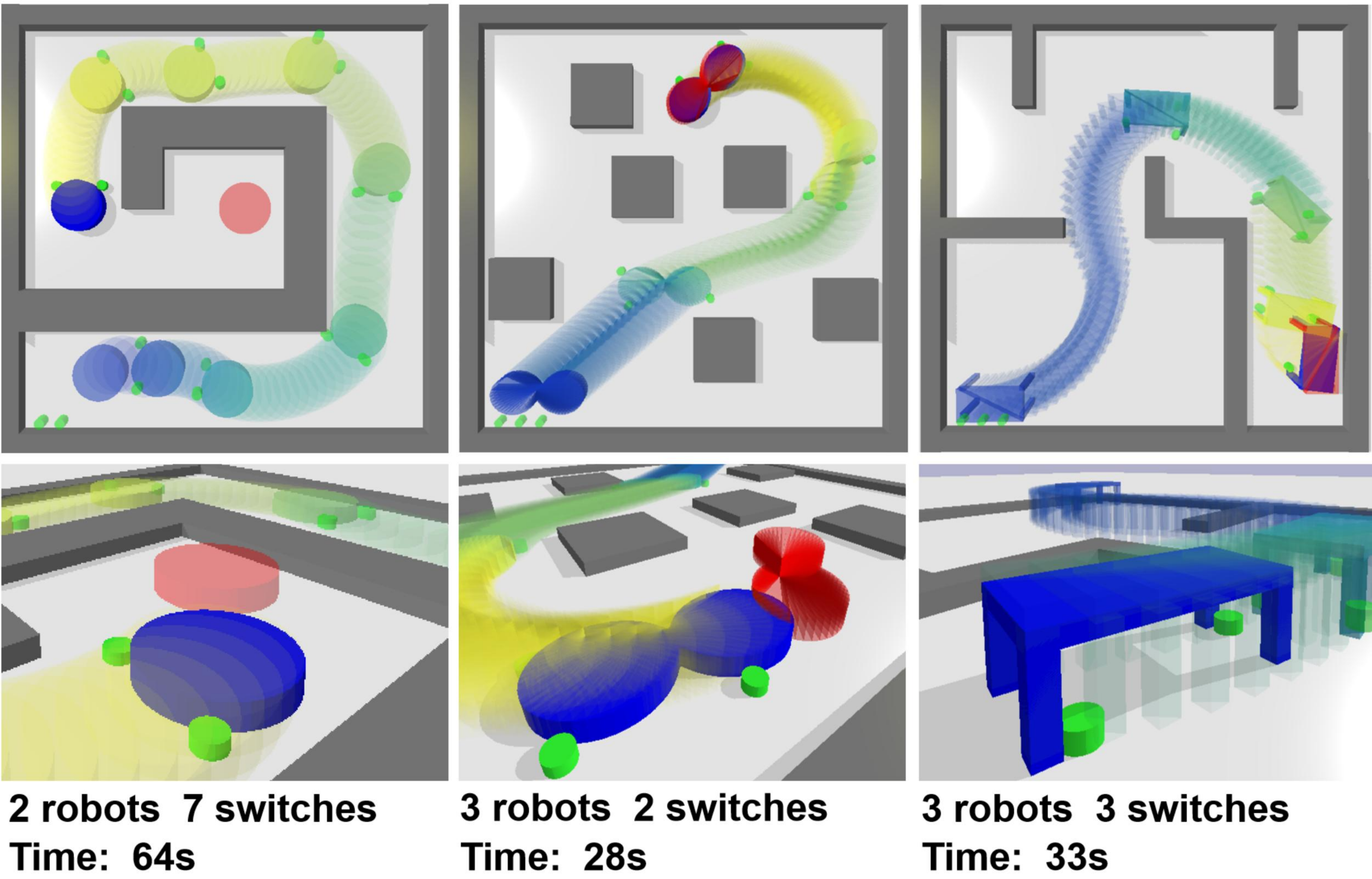}
  \vspace{-0mm}
  \caption{\change{Extension to non-polytopic objects:
  cylinder (\textbf{Left}), 8-shape (\textbf{Middle}),
  and table (\textbf{Right}).}}
  \label{fig:general_object}
  \vspace{-7mm}
\end{figure}

\subsubsection{\change{Generalization}}
\change{As discussed in Sec.~\ref{subsubsec:general},
several notable generalizations of the proposed method
are evaluated.}

\change{{(I) \textbf{Non-polytopic objects}}.
As shown in the Fig.\ref{fig:general_object},
the proposed approach is tested with three objects with general shape:
cylinder, 8-shape, and a table, i.e., non-polytopic objects.
To obtain all possible contact points for general shapes,
we traverse all possible positions of robots near the 3D boundaries of the object
with a certain discretization precision,
and further refine these points to ensure close contact with the object
via 3D collision checking.
Then, the same planning algorithm is applied for~$2,3,3$ robots.
The resulting hybrid plan and trajectories
are shown in~Fig.\ref{fig:general_object}:
the cylinder is pushed to the goal with~$8$ modes and~$64.0s$,
while the 8-shape reaches the goal with~$3$ modes and~$28.4s$.
It is interesting to note that although the table is only
partially in contact with the ground, the algorithm can find
the optimal pushing modes around four legs of the table
with $3$ mode switches.
}

\change{{(II) \textbf{Heterogeneous robots}}.
Three heterogeneous robots are deployed for the same pushing task
as in Fig.~\ref{fig:robotnum}, i.e.,
the maximum pushing forces~$f_{n,\max}$ in~\eqref{eq:force-limit}
are $60N$, $40N$ and $20N$, respectively.
The assignment and planning strategy
as described in Sec.~\ref{subsubsec:general} is adopted,
completing the task with~$4$ modes and~$60.8s$.
Compared with the nominal case where all robots are identical,
the assignments of robots to contact points are significantly
different:
(i) The robot with the largest force is always
assigned to the contact point that requires the largest force
(the top-left corner);
(ii) the robot with the smallest force is assigned to the
assistive contact points (the bottom-right corner).
Lastly, if this heterogeneity is neglected,
the task can still be completed but with a much longer time of~$82.2s$.
}

\begin{figure}[t!]
  \centering
  \includegraphics[width=0.95\linewidth]{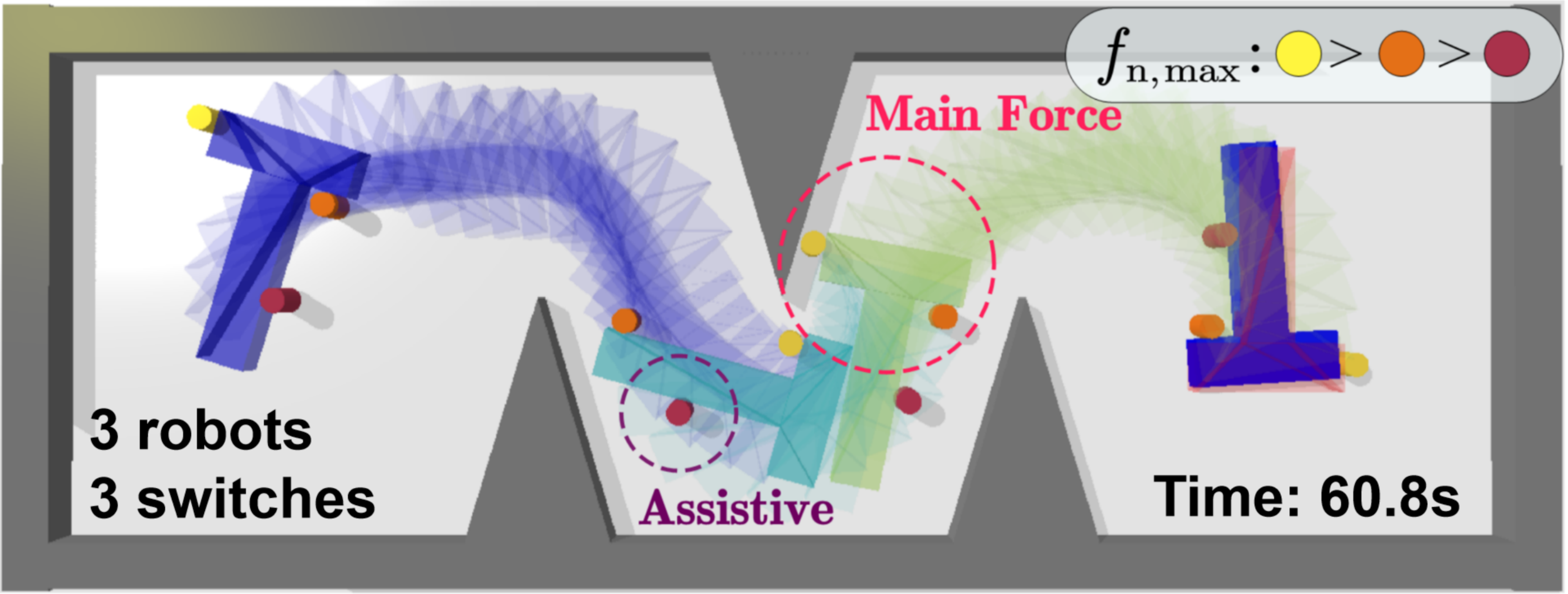}
  \vspace{-0mm}
  \caption{\change{Pushing task with $3$ heterogeneous robots with
  different maximum forces.}}
  \label{fig:hetergeneous}
  \vspace{-7mm}
\end{figure}

{(III) \textbf{6D pushing in 3D workspace}}.
The extension to pushing tasks in complex 3D workspace is demonstrated
here with a L-shape object and different number of robots.
The object has a mass of $5kg$ with the resistance coefficient~$\mathbf{K}=100$,
while the robots have a radius of $0.15m$ and a maximum pushing force of $20N$.
As shown in Fig.~\ref{fig:3d_scenarios},
the workspace has a size of~$10m\times 20m\times 5m$ with two walls and a small window.
The first row shows the snapshots of~$4$ robots performing the task,
which requires $10$ modes and a total mission time of~$92.2s$.
In contrast, the second row shows the results for $6$ robots,
which requires $4$ modes and a total mission time of~$64.3s$.
Thus, similar to the 2D case, a larger fleet often requires
less number of mode switches and faster completion.

\subsubsection{Comparisons}
The following \textbf{four} baselines are compared with the
proposed method (\textbf{KG-HS}):

(I) \textbf{POCB}, which follows the occlusion-based method
proposed in~\cite{chen2015occlusion},
i.e., assuming that a light source is placed at the goal point,
each robot randomly selects a contact point on the occluded surface
and applies a certain amount of pushing force.
Since the original algorithm only applies to star-shaped objects
and mainly for free-space
(intermediate goal points are set manually for cluttered environments),
two modifications are made for the comparison:
(i) the occlusion condition is replaced with an angular condition,
i.e., $\Delta \theta_{\texttt{G},\texttt{n}}\triangleq
\theta_{\texttt{G}}-\theta_{\texttt{n}}\in [-\frac{\pi}{2},\frac{\pi}{2}]$,
where $\theta_{\texttt{G}}$ is angle between the target center
and the goal;
and $\theta_{\texttt{n}}$ is the angle of the normal vector
at the contact point;
and (ii) the~$\texttt{A}^\star$ algorithm is applied to generate
a guiding path consisting of intermediate goal points;

\begin{table}[t!]
  \begin{center}
  \caption{Comparison with four baselines.}\label{tab:comparison}
  \begin{threeparttable}
  \begin{tabular}{cccccccc}
    \toprule[1pt]
    \textbf{Algorithm} & \textbf{SR}\tnote{1} & \textbf{TE}\tnote{2}& \textbf{CC}\tnote{3}
    & \textbf{SM}\tnote{4} & \textbf{PT}\tnote{5} & \textbf{ET}\tnote{6} & \textbf{EE}\tnote{7}\\
    \midrule
    \textbf{KG-HS} & 1.00 & 0.03 & 2241 & 0.31 & 13.21 & 44.31 & 0.14\\
    \midrule
    \textbf{AUS} & 1.00 & 0.03 & 2521 & 0.52 & 12.28 & 51.29 & 0.14\\
    \midrule
    \textbf{SDF} & 0.98 & 0.05 & 3369 & 0.58 & 20.88 & 60.18 & 0.14\\
    \midrule
    \textbf{POCB} & 0.57 & 1.94 & 6066 & 1.86 & 0.15 & 128.47 & 1.61\\
    \midrule
    \textbf{IAB} & 0.65 & 1.82 & 5130 & 0.73 & 3.56 & 117.31 & 1.70\\
    \bottomrule[1pt]
    \end{tabular}
  \begin{tablenotes}
  \item[1] Success rate; \item[2] Tracking error; \item[3] Control cost;
  \item[4] Smoothness; \item[5] Planning time; \item[6] Execution time; \item[7] End Error.
  \end{tablenotes}
  \end{threeparttable}
  \end{center}
  \vspace{-4mm}
  \end{table}

(II) \textbf{IAB}, which enhances {POCB} by
introducing a probability of choosing a contact point,
which is inversely correlated with the angle~$\Delta\theta_{\texttt{G},\texttt{n}}$.
Thus, the robot can prioritize pushing directions
that are more aligned towards the target,
which is particularly significant for long rectangular objects
in Fig.~\ref{fig:baseline}.

(III) \textbf{SDF}, which replaces MDF in KGHS with single directional feasibility
in both mode generation and cost estimation;

(IV) \textbf{AUS}, which does not follow the hybrid search framework,
but divides the guiding path equally into $L$ segments until
the resulting segmented arc trajectories are collision-free with obstacles,
as described in Theorem~\ref{thm:feasibility}.

The tests are repeated~$50$ times
for each method in all three scenarios,
by choosing random initial and goal positions for the target object.
Sample trajectories are shown in Fig.~\ref{fig:baseline}
for the first scenario.
As shown in Table~\ref{tab:comparison},
\textbf{seven} metrics are compared.
It can be seen that our {KG-HS} algorithm achieves the best performance
in all metrics except for the planning time,
as both {POCB} and {IAB} requires no optimization,
while {AUS} performs less times of mode generations.
Nonetheless, the KG-HS algorithm has the shortest runtime $57.5s$
by summing up the planning and execution time.

\begin{figure}[t!]
  \centering
  \includegraphics[width=0.95\linewidth]{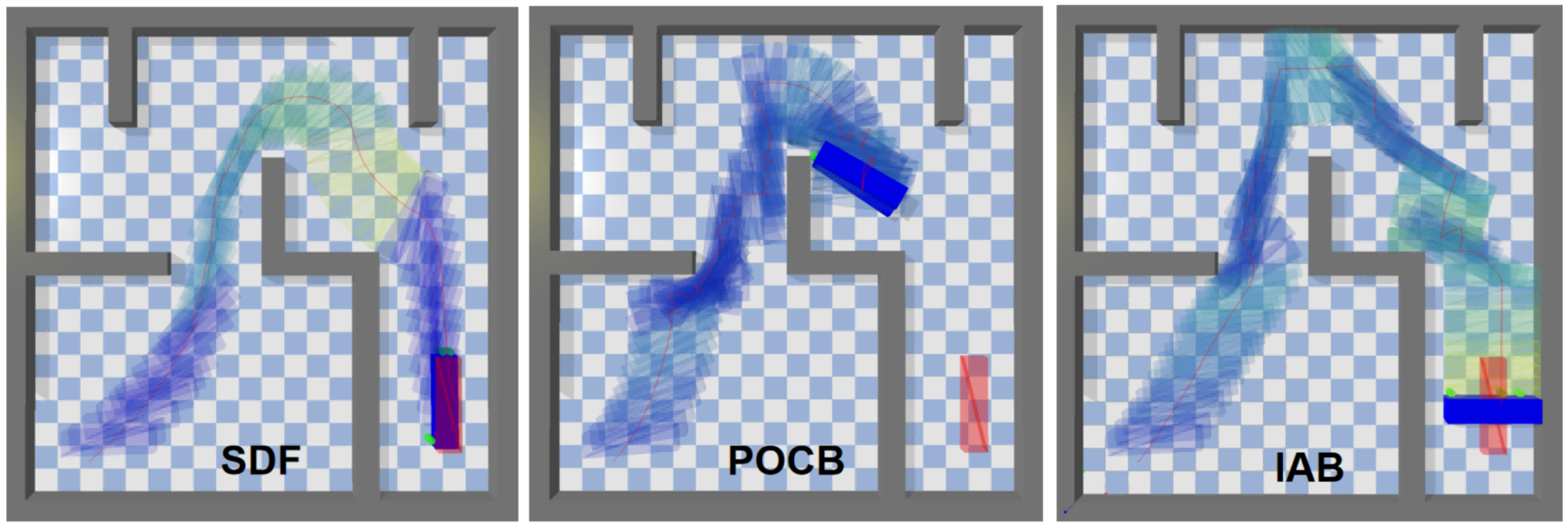}
  \vspace{-0mm}
  \caption{Trajectories of baselines
  for the first scenario.}
  \label{fig:baseline}
  \vspace{-6mm}
\end{figure}

More specifically, consider the following three aspects:
(I) \textbf{Importance of mode generation}.
It is worth noting that both AUS and SDF still outperform the POCB and IAB,
mainly due to the hybrid search framework
and the proposed feasibility analysis.
Particularly,
both POCB and {IAB} only take into account the difference between
the pushing directions and the goal direction,
thus completely neglecting the different torques generated at the contact points.
Thus, the target object can only be probabilistically moved to the target position,
rather than a continuous smooth trajectory,
especially for the first scenario where precise orientations are crucial.
In contrast, both {SDF} and KG-HS
consider the expected rotation of the object as determined by the arc transition.
(II) \textbf{Effectiveness of the hybrid search framework}.
The difference in tracking error between {AUS} and {KG-HS} algorithms
is small as they share the same modules for mode optimization and trajectory optimization.
However, compared with the equal segmentation in {AUS},
the hybrid search process in {KG-HS} can reduce the execution time by $13\%$ and the control cost by $11\%$,
while improving the smoothness by $42\%$.
(III) \textbf{Robustness against disturbances}.
For the first scenario,
random disturbances are added to the target velocity at a frequency of $10Hz$
with a variance proportional to its actual velocity at each time step.
The execution results are summarized in Fig.~\ref{fig:noise},
which shows that a significantly smaller variance in the tracking error
is achieved by {KG-HS} compared to the other methods.
\begin{figure}[t!]
  \centering
  \includegraphics[width=0.95\linewidth]{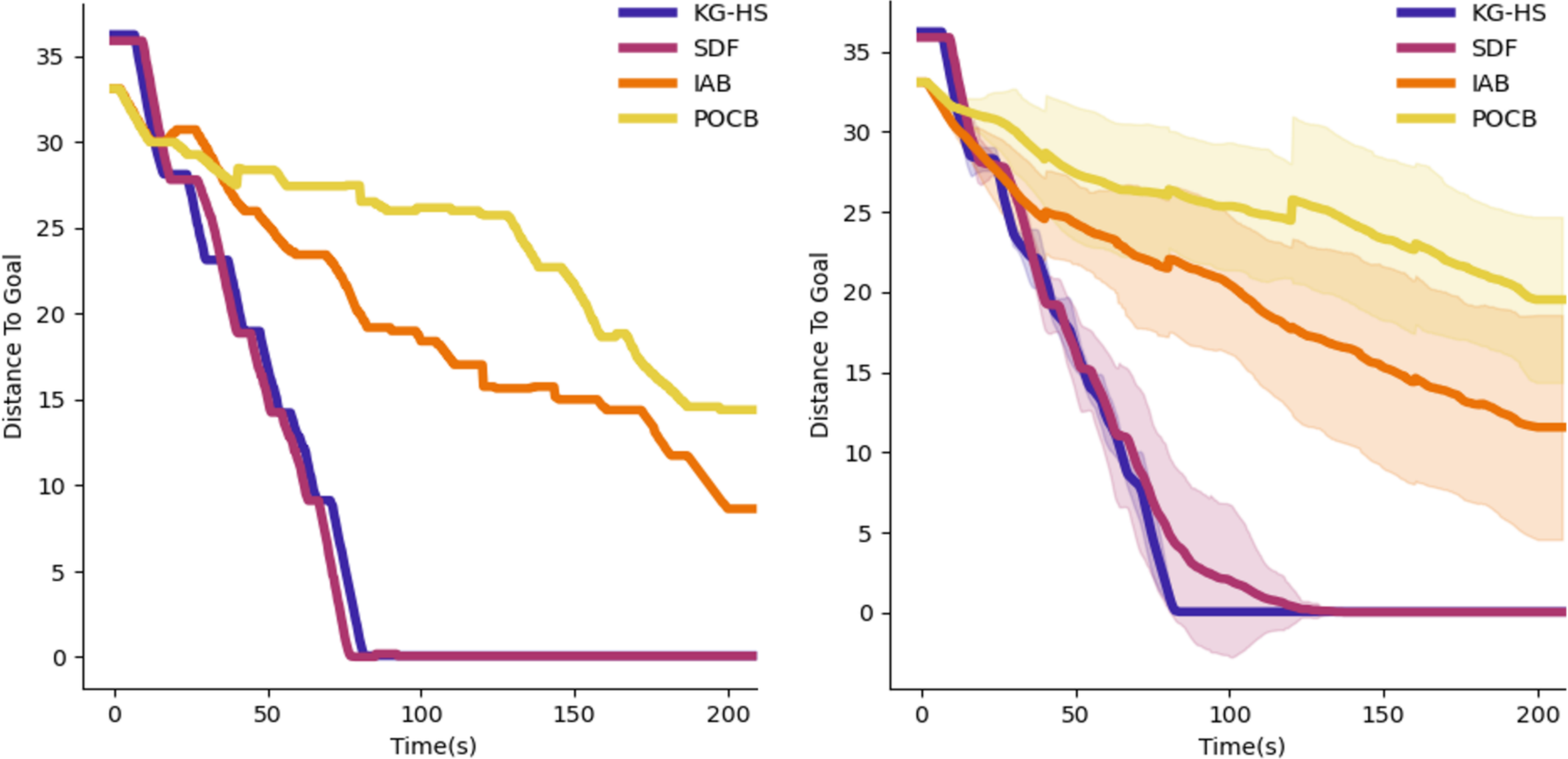}
  \vspace{-2mm}
  \caption{Distance to goal during execution without (\textbf{Left})
  and with (\textbf{Right}) disturbances,
  where the shaded area indicates the standard deviation over~$50$ runs.}
  \label{fig:noise}
  \vspace{-6mm}
\end{figure}

\subsection{Hardware Experiments}\label{subsec:hardware}

\subsubsection{System Description}\label{subsec:exp-description}
\change{
As shown in Fig.~\ref{fig:hardware-exp},
the hardware experiments are conducted in a $4m\times 5m$ environment,
with the motion capture system \texttt{OptiTrack}
providing global positioning of the target and the robots.
Three homogeneous ground robots \texttt{LIMO} have a rectangular shape
of~$0.2m \times 0.3m$,
with the maximum pushing force of about~$10N$.
Each robot is equipped with a NVIDIA Jetson Nano
that runs the \texttt{ROS1} system for communication and control.
The centralized planning and control algorithm is implemented on a laptop
with an Intel Core i7-1280P CPU,
which runs the master node of \texttt{ROS1} and
sends the velocity commands to all robots via wireless.
}
Moreover, a L-shape object with a mass of $1.5Kg$
is adopted for the second scenario,
while a rectangular object with a mass of $1Kg$ for the first scenario.
\change{
Both objects have a ground friction coefficient of $0.5$,
and a side friction coefficient of $0.3$.
}
The obstacles are made of white cardboards at known positions.
The main challenge for the second scenario
lies in the non-convex shape of the target with a relatively larger mass.
The object shape in the first scenario is simpler
but with tighter geometric constraints such as narrow passages.

 \begin{figure}[t!]
  \centering
  \includegraphics[width=1.0\linewidth]{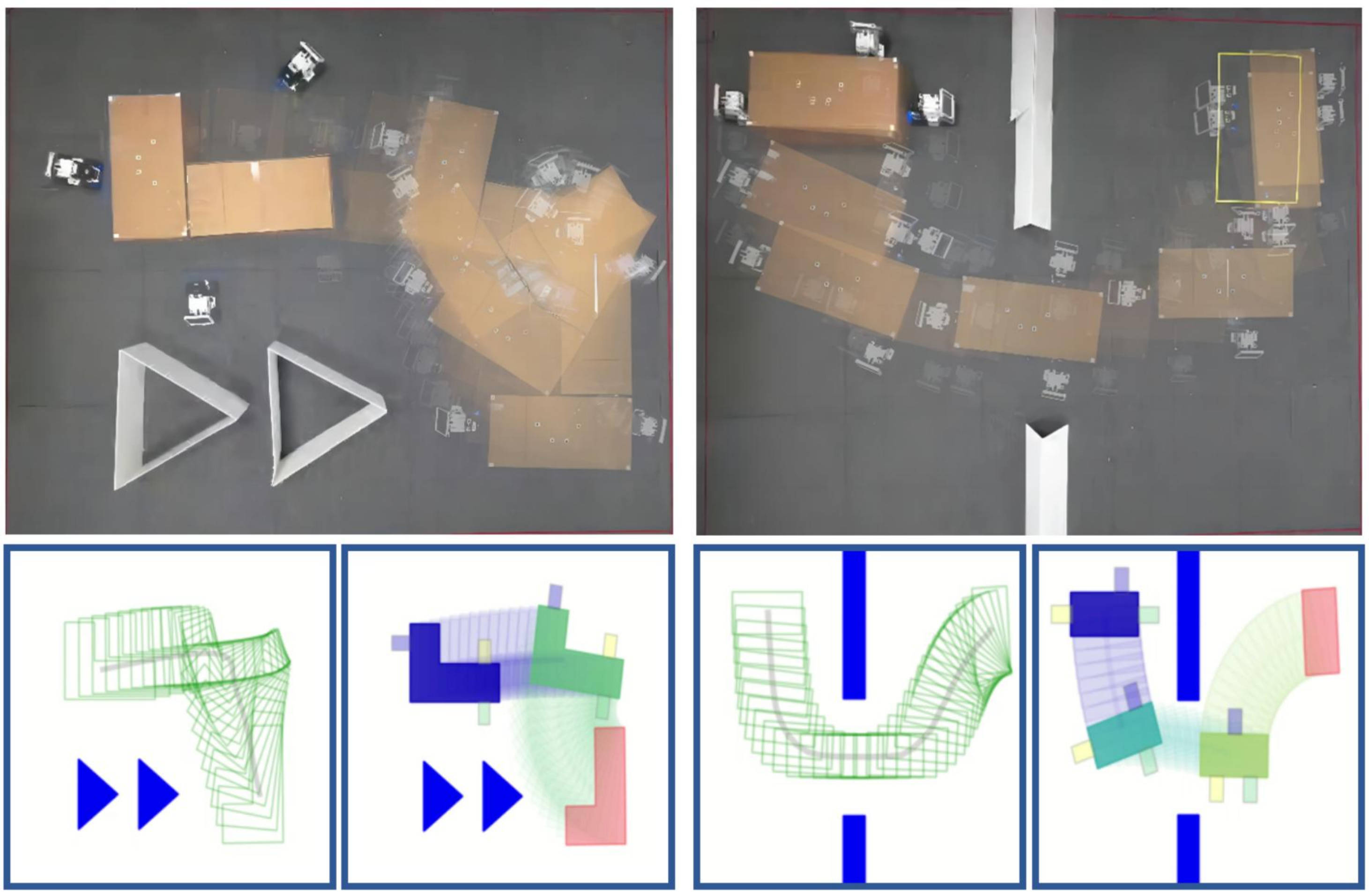}
  \vspace{-3mm}
  \caption{Two scenarios of the hardware experiments (\textbf{Top}),
  the guiding path and the hybrid plan (\textbf{Bottom}).}
  \label{fig:hardware-exp}
  \vspace{-6mm}
\end{figure}

\subsubsection{Results}\label{subsec:exp-results}
Overall, the experiment results are consistent with the simulations
for both free and cluttered environments.
Namely, via the proposed method, both tasks are completed successfully
within $50s$ and $55s$, which require $2$ and $3$ modes, respectively.
The average tracking errors are $0.023m$ and $0.052m$.
\change{Compared with simulations,
  the model uncertainties caused by communication delays
  and actuation noises or slipping are much more apparent in the hardware experiments.
  Thus, the replanning mechanism is necessary during online execution,
  as described in Sec.~\ref{subsubsec:adaptation}.
Once the object deviates from the planned trajectory by more than $0.2m$,
or remains static for~$5s$,
the KG-HS algorithm is re-executed to generate a new hybrid plan,}
which is then tracked by the NMPC controller.
As shown in Fig.~\ref{fig:hardware-replanning},
the system in some cases can not recover from the large deviation
from the reference trajectory in the current interaction mode.
Thus, a replanning is triggered,
after which another mode is adopted.

\section{Conclusion} \label{sec:conclusion}
This work proposes a collaborative planar pushing strategy of polytopic objects
with multiple robots in complex scenes.
It utilizes a hybrid search algorithm to generate intermediate keyframes
and a sparse optimization to generate sufficient pushing modes.
It is shown to be scalable and effective for any polytopic object
and an arbitrary number of robots in obstacle-cluttered scenes.
The completeness and feasibility of our algorithm have been demonstrated
both theoretically and experimentally.
Extensive numerical simulation and hardware experiments validate its
effectiveness for collaborative pushing tasks,
with numerous extensions to more general scenarios.

\change{
There are some limitations of the proposed approach that are worth mentioning:
(i) The quasi-static analyses in Sec.~\ref{subsubsec:quasi-static} are only
valid for slow movements of the object and the robots.
In high-speed scenarios,
the object may experience large instantaneous accelerations,
which could require the robotic fleet to react
with larger forces and torques that may not be allowed
in the current pushing mode.
Although this can be partially mitigated by
the multi-directional feasibility estimation
proposed in~\eqref{eq:multi-directional-feasibility}
and the inertia term to NMPC in~\eqref{eq:nmpc},
a comprehensive solution for the high-speed scenarios
remains part of our future work;
(ii) The intrinsics of the object such as mass, center of gravity
and friction coefficient
are all assumed to be known for the evaluation of MG-SO
in~\eqref{eq:multi-directional-feasibility}.
Although slight deviations can be handled by the NMPC in~\eqref{eq:nmpc}
and the re-planning process as described in Sec.~\ref{subsubsec:adaptation},
large uncertainties might require other techniques,
e.g., model-free and learning-based methods that adapt pushing strategies online.
This is part of our ongoing work;
(iii) The workspace is assumed to be feasible in the sense that
a guiding path~$\overline{\mathbf{S}}$ always exists for the given
initial and goal states.
Despite of being rather mild, it can be further relaxed,
e.g., by allowing movable obstacles or by temporarily
pulling out some robots to make room and re-join later,
which is part of future work;
(iv) Lastly, a centralized perception system is assumed that has a full knowledge
of the system state and the environment.
A decentralized method for perception and state estimation
would further improve the applicability.
}
\begin{figure}[t!]
    \centering
    \includegraphics[width=1.0\linewidth]{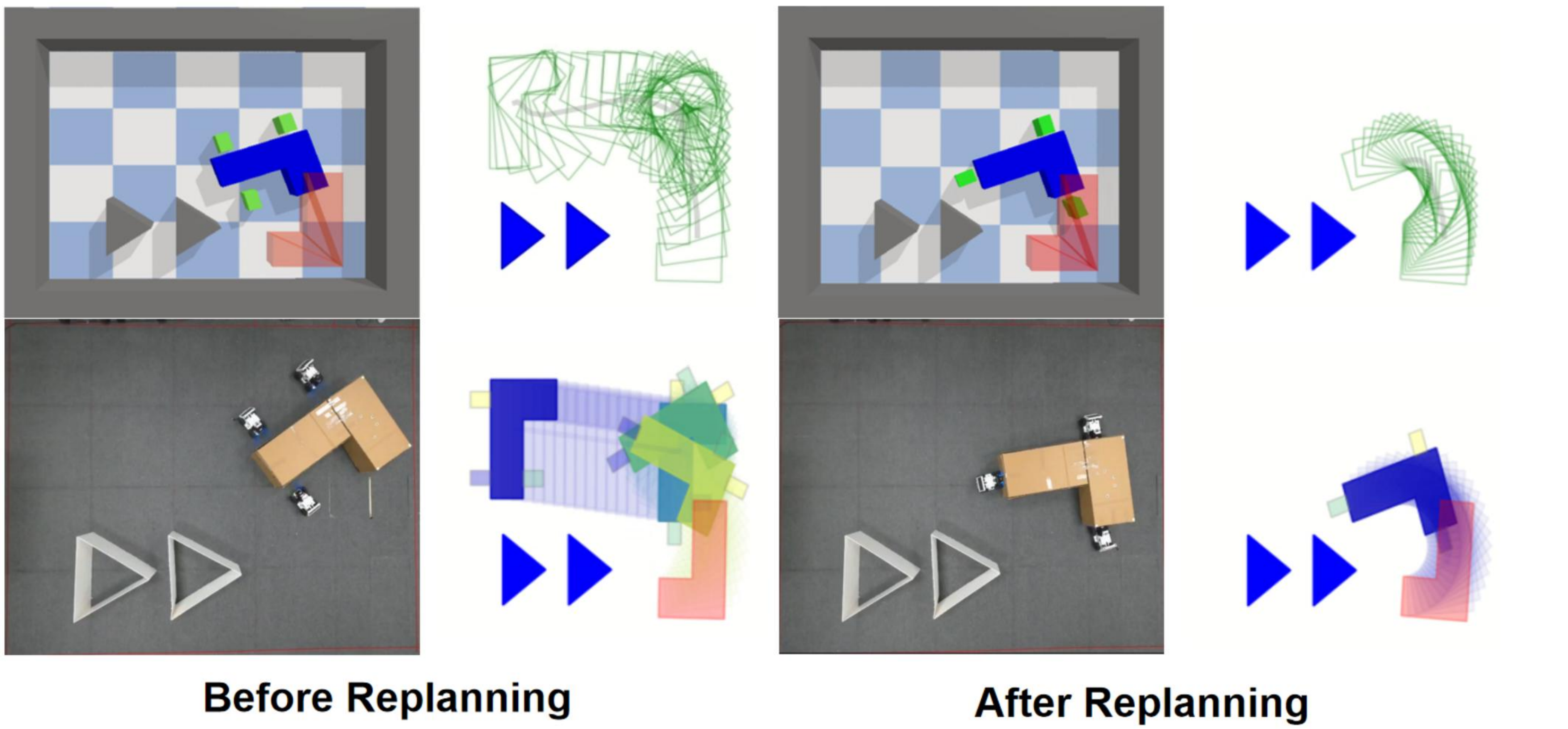}
    \vspace{-4mm}
    \caption{The hybrid plan and the resulting trajectory before (\textbf{Left})
    and after (\textbf{Right}) replanning.}
    \label{fig:hardware-replanning}
    \vspace{-6mm}
\end{figure}
  
\section*{Acknowledgement}
This work was supported by the National Key Research and Development Program of China under grant 2023YFB4706500;
the National Natural Science Foundation of China (NSFC) under grants 62203017, U2241214, T2121002;
and the Fundamental Research Funds for the central universities.

\bibliographystyle{plainnat}
\bibliography{contents/references}

\end{document}